\documentclass[final,12pt]{clear2026} %

\title[DDCDs]{Smoothing the Landscape: Causal Structure Learning via Diffusion Denoising Objectives}

\usepackage{times}
\usepackage{longtable}
\usepackage{booktabs}

\usepackage{amsmath,amsfonts,bm}

\def\eqref#1{equation~\ref{#1}}

\def\1{\bm{1}}

\def\vt{{\bm{t}}}

\def\vv{{\bm{v}}}

\def\vx{{\bm{x}}}

\def\vz{{\bm{z}}}

\def\mE{{\bm{E}}}

\def\mI{{\bm{I}}}

\def\mW{{\bm{W}}}
\def\mX{{\bm{X}}}
\def\mY{{\bm{Y}}}
\def\mZ{{\bm{Z}}}

\DeclareMathAlphabet{\mathsfit}{\encodingdefault}{\sfdefault}{m}{sl}
\SetMathAlphabet{\mathsfit}{bold}{\encodingdefault}{\sfdefault}{bx}{n}

\def\gG{{\mathcal{G}}}

\def\gN{{\mathcal{N}}}

\newcommand{\R}{\mathbb{R}}

\author{%
\Name{Hao Zhu}\thanks{Work done while at the Department of Computer Science, Tufts University} \Email{haozhu233@gmail.com}\\
\addr Beth Israel Deaconess Medical Center\\
Harvard Medical School \\
Boston, MA, USA
\AND
\Name{Di Zhou}
\Email{Di.Zhou@tufts.edu}\\
\addr Department of Computer Science\\
Tufts University\\
Medford, MA, USA
\AND
\Name{Donna Slonim}
\Email{Donna.Slonim@tufts.edu}\\
\addr Department of Computer Science\\
Tufts University\\
Medford, MA, USA%
}

\begin{document}

\maketitle

\begin{abstract}
  Understanding causal dependencies in observational data is critical for informing decision-making. These relationships are often modeled as Bayesian Networks (BNs) and Directed Acyclic Graphs (DAGs). Existing methods, such as NOTEARS and DAG-GNN, often face issues with scalability and stability in high-dimensional data, especially when there is a feature-sample imbalance. Here, we show that the denoising score matching objective of diffusion models could smooth the gradients for faster, more stable convergence. We also propose an adaptive k-hop acyclicity constraint that improves runtime over existing solutions that require matrix inversion. We name this framework Denoising Diffusion Causal Discovery (DDCD). Unlike generative diffusion models, DDCD utilizes the reverse denoising process to infer a parameterized causal structure rather than to generate data. We demonstrate the competitive performance of DDCDs on synthetic benchmarking data. We also show that our methods are practically useful by conducting qualitative analyses on two real-world examples. Code is available at this url: \href{https://github.com/haozhu233/ddcd}{https://github.com/haozhu233/ddcd}.
\end{abstract}

\begin{keywords}%
  causal discovery, structural equation models, denoising models%
\end{keywords}

\section{Introduction}

Learning network structures representing causal dependencies from high-dimensional observational data has long been a crucial goal.  It plays an important role in multiple disciplines, including genetics, epidemiology, and economics \citep{%
pearl2009causality, uhler2013geometry,spirtes2016causal,pingault2018using, glymour2019review}, and plays a growing role in healthcare research \citep{slonim2002patterns,chandak2022building}.

Such networks, where nodes represent feature variables and edges capture potential relationships, are often represented as Directed Acyclic Graphs (DAGs), which allow no cycles. There has been a great deal of prior work developing methods to address this problem. For example, the PC algorithm \citep{spirtes1991algorithm} is a constraint-based approach that iteratively tests conditional independence, GES \citep{chickering2002optimal} searches for the causal structure  maximizing a scoring function, and LiNGAM \citep{shimizu2012discovery}  relies on the structural equation model (SEM). While these early methods face scalability challenges, subsequent work has extended constraint-based and score-based approaches to high-dimensional settings, including methods for latent confounders \citep{ogarrio2016hybrid,colombo2012learning,pal2025penalized,ribeiro2025dcfci} and efficient Bayesian network structure learning \citep{zhu2021efficient}.

In the continuous optimization paradigm, NOTEARS \citep{zheng2018dags} introduced a continuous acyclicity score that can be solved by a regular numerical optimizer. 
Many other studies have since extended this formulation, focusing on scalability \citep{yu2019dag,lee2019scaling,ng2020role,yu2021dags}, convexity \citep{bello2022dagma,deng2023optimizing,ng2024structure}, sparsity control \citep{wei2020dags,ng2020role}, extension to interventional data \citep{brouillard2020differentiable}, and non-linearity \citep{yu2019dag, ng2019graph,zheng2020learning, yang2021causalvae,shen2022weakly, ng2022masked,kalainathan2022structural,lachapelle2019gradient}. Variations of this model have also been applied to a wide range of settings, such as time series \citep{pamfil2020dynotears,sun2021nts,shang2021discrete} and gene networks \citep{shu2021modeling, agamah2022computational,regdiff}. There are also discussions of the application of such models to datasets with unequal variances and different data types \citep{reisach2021beware,kaiser2021unsuitability,ng2024structure}. Methods that focus on topological ordering instead of a structural DAG constraint have also been explored \citep{sanchez2022diffusion}. 

We propose Denoising Diffusion Causal Discovery (DDCD), inspired by the forward diffusion perturbation and denoising score matching objective of DDPMs \citep{ddpm}. Unlike generative DDPMs, DDCD repurposes the reverse process to learn structural parameters that best denoise the data. Our contributions are the following:

\begin{itemize}
  \item We prove that the denoising score matching objective could be used for causal structural. We  inference and show that it helps avoid sharp local minima and encourage smooth gradients. 
  
  \item We propose an adaptive k-hop acyclicity constraint that transitions from local neighborhood to global enforcement, significantly reducing runtime while guaranteeing valid DAG recovery. 
  
  \item We introduce a permutation-invariant batch sampling strategy that decouples optimization complexity from sample size, ensuring consistent convergence and scalability. 

  \item We propose DDCD-Smooth, which addresses the ``varsortability'' problem in continuous causal discovery by normalizing features to equal scales. This prevents methods from exploiting variance differences as an artifact for identification, making the approach robust to heterogeneous feature scales in real-world data.

\end{itemize}

\section{Background and Related Work} \label{sec:background}

\subsection{Problem Statement} \label{subsec:ps}
Given a dataset $\mX \in \R^{n \times d}$, %
with $n$ samples and $d$ feature variables,
the objective is to learn a dependency graph $\gG$ represented by the weighted adjacency matrix $\mW \in \R^{d \times d}$. Such a graph is often defined as a Bayesian Network (BN), which permits no cycles. 

To ensure the identifiability of $\gG$ and the validity of our proposed denoising objective, we make the following standard assumptions common in causal discovery literature \citep{spirtes2000causation,pearl2009causality}: \textbf{Assumption 1 (Causal Sufficiency)}: there are no unobserved confounders influencing multiple observed variables. \textbf{Assumption 2 (No Selection Bias)}: the samples in $\mX$ are drawn i.i.d. from the underlying distribution with no selection bias. \textbf{Assumption 3 (Acyclicity)}: the true causal structure $\gG$ is a DAG. \textbf{Assumption 4 (Additive Noise Model)}: the data follows a structural equation model $x_j = f_j(PA_j)+z_j$ with additive noise. \textbf{Assumption 5 (Identifiability)}: The error terms $z$ are jointly independent. For unique identifiability beyond the Markov Equivalence Class, we assume errors are either non-Gaussian \citep{shimizu2006linear} or Gaussian with equal variances \citep{peters2014identifiability}. \textbf{Assumption 6 (Faithfulness)}: The joint distribution P(X) is faithful to $\gG$.

\subsection{Structural Equation Models (SEMs) and Continuous DAG constraint} \label{sem}

Structural Equation Models (SEMs) \citep{grace2006structural,shipley2016cause,kline2023principles} provide a framework to model variable dependencies.  For a linear SEM, we simply assume that each variable is a linear combination of its parents with some noise. In its matrix multiplication form, we have

\begin{equation}
   \label{eqn:linear_sem}
   \mX = \mX\mW + \mE,
\end{equation}

\noindent where $\mE \in \R^{n \times d}$ captures the error terms. Based on this assumption, many existing SEMs aim to estimate matrix $\mW$ such that the reconstruction error is minimized \citep{van20130,loh2014high,zheng2018dags}. Since the adjacency matrix is often sparse, many methods choose to add either L1 or L2 regularization on $\mW$ to encourage sparsity \citep{vowels2022d}. In this case, we have the following training objective,

\begin{equation}
   \label{eqn:linear_sem_obj}
   \min_{W} \frac{1}{2n}\|\mX-\mX\mW\|^2_F + \lambda_1\|\mW\|_1 + \lambda_2\|\mW\|_2.
\end{equation}

While traditional approaches often rely on combinatorial optimization, \citet{zheng2018dags} proposed NOTEARS, which introduced a continuous score characterizing graph acyclicity:

\begin{equation}
   \label{eqn:notears_dag}
   h(\mW) = \text{tr}(e^{\mW \circ \mW})-d,
\end{equation}

\noindent where $\circ$ is the Hadamard product, $e^{\mW}$ is the matrix exponential of $\mW$, and $\text{tr}()$ is the trace of a matrix. Essentially, matrix $\mW$ is a DAG if and only if $h(\mW) = 0$. Since the function $h(\mW)$ has a simple and smooth gradient function, it can be used in many gradient-based continuous optimization algorithms. 
In terms of complexity, since the score function $h(\mW)$ requires the matrix exponential, the runtime of NOTEARS is at least $\mathcal{O}(d^3)$ \citep{arioli1996pade}.

\subsection{Denoising Diffusion Probabilistic Models}
Denoising Diffusion Probabilistic Models (DDPMs) are a class of generative model that shows strong performance in modeling complex data distributions \citep{ddpm}. A typical DDPM starts from a non-parameterized forward diffusion process. Given an unperturbed input $\vx_0$, the forward process aims to generate a series of noisy samples $\vx_0,\vx_1,...,\vx_T$ over $T$ steps, where $\vx_T$ usually stands for pure noise. In each step, a small amount of Gaussian noise is gradually introduced following a diffusion schedule $\beta$ as shown in Equation \ref{eqn:diffusion_forward}.  

\begin{equation}
   \label{eqn:diffusion_forward}
   \vx_t = \sqrt{1-\beta_t}\vx_{t-1} + \sqrt{\beta_t}\vz_{t-1}
\end{equation}

\noindent Here, $\vz_{t-1} \sim \gN(0, 1)$, so Equation \ref{eqn:diffusion_forward} is essentially trying to reduce the means to $\mathbf{0}$ while increasing the variances  to $\mathbf{1}$. With %
reparameterization, 
it can be rewritten as Equation \ref{eqn:diffusion_forward_2}: %

\begin{equation}
   \label{eqn:diffusion_forward_2}
   \vx_t = \sqrt{\overline{\alpha_t}}\vx_{0} + \sqrt{1-\overline{\alpha_t}}\vz,
\end{equation}

\noindent where $\overline{\alpha_t} = \prod_{i=0}^{t} (1-\beta_i)$ and $\vz \sim \gN(0, 1)$, so the noisy data $x_t$ can be generated in one step. 

The actual modeling piece of DDPM is the reverse model, which is trained to predict the added noise $\vz$ and to denoise the data. The  choice of model for the reverse process depends on the input data. Recent studies have suggested similarity between diffusion models and a generalized form of variational autoencoder (VAE) \citep{kingma2013auto} with infinite latent spaces \citep{luo2022understanding}. 

\section{Methods} \label{method}

\subsection{Denoising Objective for Linear SEMs \& NOTEARS}

Inspired by the denoising objective in DDPMs, we propose an alternative training objective for learning the adjacency matrix $\mW$ in SEMs. We augment each sample in the same way as the forward diffusion process in Equation \ref{eqn:diffusion_forward_2}. Then, we will optimize a reverse model with the parameterized $\mW$ to predict the added noise. Note that unlike standard diffusion models where the reverse model learns a time-dependent transition kernel, our reverse model focuses on the single optimal $\mW$ that minimizes denoising error across all noise levels.

For linear SEMs, the reverse model is trying to minimize the following denoising objective:

\begin{equation}
   \label{eqn:denoising_obj}
   \min_{W} \frac{1}{2n}\|(\mX_\vt - \mX_\vt\mW) - \text{diag}(\sqrt{1-\overline{\alpha_\vt}})\mZ(\mI - \mW)\|^2_F \\+ \lambda_1\|\mW\|_1 + \lambda_2\|\mW\|_2.
\end{equation}

Here, following diffusion literature, we use $t \in \{1, \ldots, T\}$ to denote diffusion time steps, where $T$ is the total number of diffusion steps. The original (unperturbed) data is $\mX_0$, and $\mX_\vt$ denotes data perturbed at diffusion step $\vt$, which is a vector of diffusion time steps for all samples. The cumulative noise schedule is $\overline{\alpha}_t = \prod_{i=1}^{t}(1-\beta_i)$, where $\beta_i$ is the noise added at step $i$. The random noise matrix is $\mZ \sim \mathcal{N}(0, \mathbf{I})$. $\text{diag}(\vv)$ is the diagonal operator that converts vector $\vv$ to a diagonal matrix. 

\begin{theorem}[Objective Equivalence] \label{theorm1}
   For linear SEMs, minimizing the denoising objective in Eq. \ref{eqn:denoising_obj} is equivalent to minimizing the standard SEM reconstruction loss in Eq. \ref{eqn:linear_sem_obj}. Therefore, under Assumption 1-6 in  section \ref{subsec:ps}, the denoising objective can be used for causal structure learning.
\end{theorem}

\begin{proof}
   Consider the case when each sample in $\mX$ is perturbed using the forward diffusion process in \ref{eqn:diffusion_forward_2}. The perturbed observational data could be written as in Equation \ref{eqn:diffusion_forward_matrix}. Here we use $\vt$ to denote the different diffusion time steps for all samples in $\mX$ and the $\text{diag()}$ operator scales each row of $\mX_0$ and $\mZ$ accordingly based on the diffusion schedule. 

\begin{equation}
   \label{eqn:diffusion_forward_matrix}
   \mX_\vt = \text{diag}(\sqrt{\overline{\alpha_\vt}})\mX_{\textbf{0}} + \text{diag}(\sqrt{1-\overline{\alpha_\vt}})\mZ,
\end{equation}

With Equations \ref{eqn:diffusion_forward_matrix} and \ref{eqn:linear_sem}, we can develop the following equations.

\begin{equation}
   \label{eqn:diffusion_forward_matrix_w}
   \mX_\vt\mW = \text{diag}(\sqrt{\overline{\alpha_\vt}})\mX_{\textbf{0}}\mW + \text{diag}(\sqrt{1-\overline{\alpha_\vt}})\mZ\mW,
\end{equation}

\begin{equation}
   \mX_\vt - \mX_\vt \mW = \text{diag}(\sqrt{\overline{\alpha_\vt}})(\mX_{\textbf{0}} - \mX_{\textbf{0}}\mW)  + \text{diag}(\sqrt{1-\overline{\alpha_\vt}})(\mZ - \mZ\mW),
\end{equation}

\begin{equation}
   \text{diag}(\sqrt{\overline{\alpha_\vt}})(\mX_{\textbf{0}} - \mX_{\textbf{0}}\mW) = (\mX_\vt - \mX_\vt \mW) - \text{diag}(\sqrt{1-\overline{\alpha_\vt}})\mZ(\mI - \mW).
\end{equation}

Recall that in Equation \ref{eqn:linear_sem_obj}, the main objective is to minimize $\mX_{\textbf{0}} - \mX_{\textbf{0}}\mW$. The proof above shows that minimizing $\mX_{\textbf{0}}-\mX_{\textbf{0}}\mW$ is equivalent to minimizing the right-hand side, which measures the distance between $(\mX_\vt - \mX_\vt \mW)$ and $\text{diag}(\sqrt{1-\overline{\alpha_\vt}})\mZ(\mI - \mW)$. Therefore, the denoising objective in Equation \ref{eqn:denoising_obj} is algebraically equivalent to the standard SEM reconstruction objective. 

\end{proof}

\textbf{Notes on Gradient Smoothing:} Beyond objective equivalence, the denoising objective helps adjust the optimization dynamics. By minimizing loss over perturbed data $\mX_t$, we are in fact performing Randomized Smoothing \citep{duchi2012randomized,nesterov2017random}. This places an upper bound on the curvature of the loss landscape (the Lipschitz constant of the gradient) and effectively smooths the sharp local minima common in high-dimensional SEMs \citep{keskar2016large}. Consequently, the optimizer avoids volatile jumps and converges more reliably, even with limited samples. To illustrate this point, we implemented a toy model called NOTEARS-Denoising with the denoising diffusion process, while everything else is the same as in the original implementation of NOTEARS-Linear. 

\subsection{Adaptive k-hop Acyclicity Constraint with gradient clipping} \label{khop-dag}

In practice, the calculation of NOTEARS' DAG constraint involves the matrix exponential, so its runtime is at least $\mathcal{O}(d^3)$. One potential optimization here is to enforce a ``k-hop acyclicity constraint" that only checks the acyclicity score within local neighborhoods of $k$-hops. By keeping a running sum, we replace the matrix exponential with just matrix multiplication, reducing the runtime and theoretical complexity. %
The formula for the k-hop acyclicity constraint is in Equation \ref{eqn:khop_dag}, where $\gamma$ is a scaling factor. A detailed discussion of this is provided in Appendix \ref{khop_dag_app}.

\begin{equation}
   \label{eqn:khop_dag}
   h(\mW, k, \gamma) = \sum_{j=1}^{k+1}\frac{1}{j!\gamma^{2j}}\text{tr}((\gamma\mW \circ \gamma\mW)^j)
\end{equation}

However, fixing $k$ to a small constant value introduces a theoretical risk of longer-range cycles. To eliminate this risk without sacrificing training efficiency, we introduce an \textit{Adaptive Acyclicity Curriculum}. Motivated by the observation that early optimization steps on the adjacency matrix $\mW$ are noisy, we suggest that enforcing global acyclicity at the start of training is computationally wasteful. Instead, we dynamically schedule $k$ during optimization. Let $\tau$ be the current training iteration and $N_{\text{iter}}$ be the total number of iterations.\footnote{We use $\tau$ and $N_{\text{iter}}$ for training iterations to distinguish from diffusion time steps $t$ and $T_{\text{diff}}$.} We define $k_\tau$ as:

\begin{equation}
    k_\tau =
    \begin{cases}
    3 & \text{if } \tau < 0.4 N_{\text{iter}} \quad \text{(Local Alignment)} \\
    10 & \text{if } 0.4 N_{\text{iter}} \le \tau < 0.9 N_{\text{iter}} \quad \text{(Structure Refinement)} \\
    d & \text{if } \tau \ge 0.9 N_{\text{iter}} \quad \text{(Global Guarantee)}
    \end{cases}
    \label{eq:adaptive_k}
\end{equation}

This schedule proceeds in three phases: \textbf{1. Local Alignment:} In the early phase, we restrict constraints to a small neighborhood ($k=3$). This prevents tight feedback loops (e.g., $A \to B \to A$) while allowing the optimizer to rapidly traverse the loss landscape. \textbf{2. Structure Refinement:} As the adjacency matrix sparsifies, we increase $k$ to detect medium-range cycles. \textbf{3. Global Guarantee:} In the final phase, we set $k$ equal to the number of feature variables $d$. This is mathematically equivalent to the full NOTEARS constraint \citep{zheng2018dags}, guaranteeing a valid DAG, but it runs quickly because $W$ is usually sparse at this stage.

\noindent This strategy ensures our final output is a theoretically-valid DAG, while maintaining the practical computational efficiency of the $k$-hop approximation for the vast majority of the training process.

\subsection{Permutation-Invariant Batch Sampling}

Existing causal discovery algorithms often operate on the full data matrix $\mX \in \R^{n \times d}$. To improve scalability and stability across varying dataset sizes, we employ a fixed-size bootstrap sampling strategy from RegDiffusion \citep{regdiff}. This approach is theoretically grounded in the `Deep Sets' perspective \citep{zaheer2017deep}. Since causal structure learning is permutation invariant in terms of samples, we could treat data as an exchangeable set. This allows us to use fixed-size mini-batches with resampling as uniform Monte Carlo estimators of the underlying population. At the same time, we also decouple algorithm runtime from the total sample size $n$ and gain more stable and consistent behavior on datasets of different sizes. 

\subsection{DDCD Models}
With the new denoising objective in mind, we propose a set of 3 DDCD models: DDCD Linear, DDCD Nonlinear, and DDCD Smooth. All three models, as illustrated in Figure \ref{fig:model_architecture}, are trained with the denoising objective, the diffusion training framework, and the k-hop acyclicity constraint. 

\begin{figure}[t]
\centering
\includegraphics[width=0.7\linewidth]{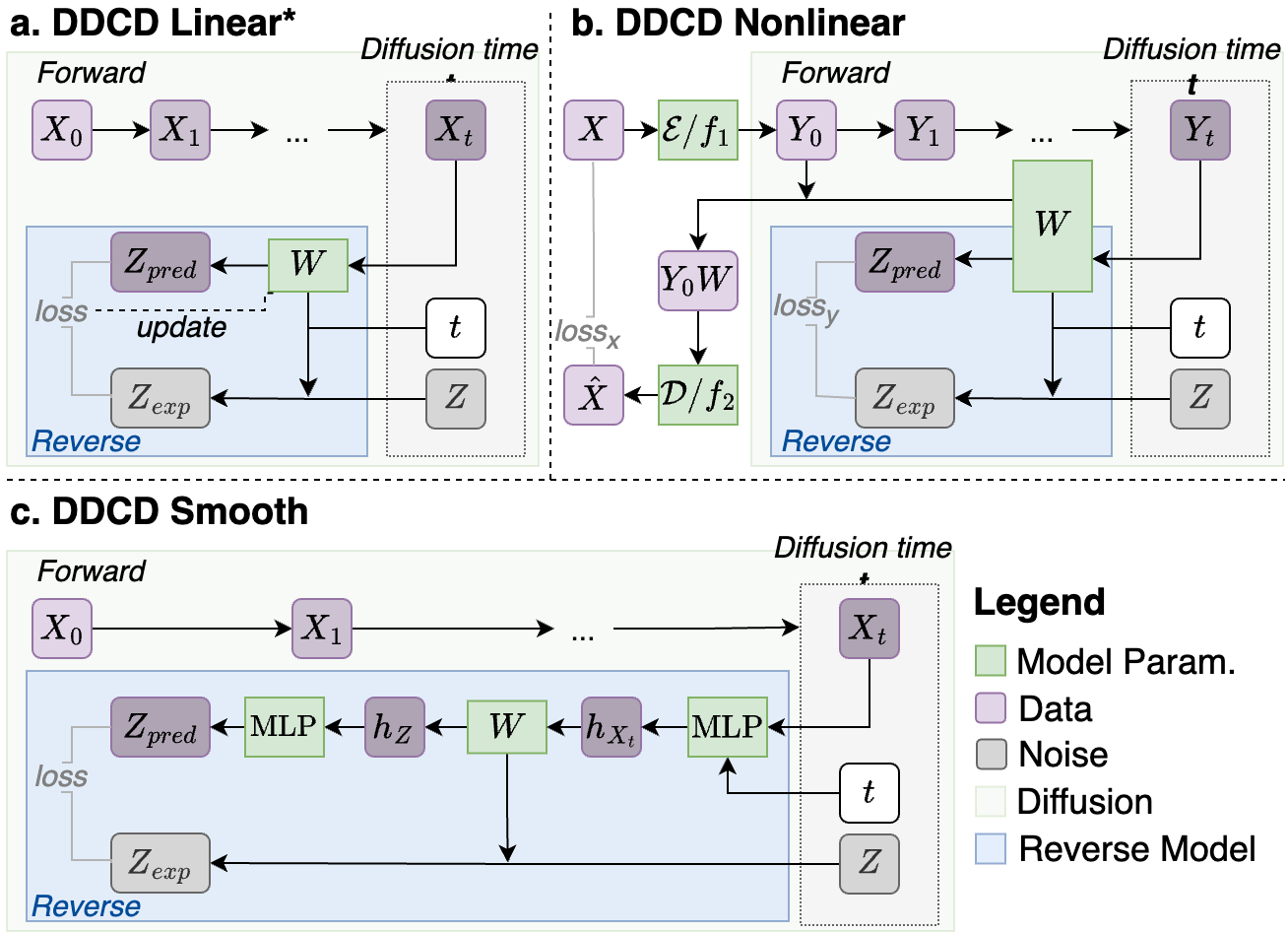}
\caption{Model architectures of proposed models in this paper. Note that the reverse path does not generate new samples and the only focus is to learn W via the denoising objective.}
\label{fig:model_architecture}
\end{figure}

\subsubsection{DDCD Linear}
The only trainable parameters in DDCD-Linear are values in the adjacency matrix $\mW$. The input of the model includes the perturbed $\mX_\vt$, the diffusion variance schedule $\sqrt{1-\overline{\alpha_\vt}}$, and the noise term $\mZ$. All three inputs are generated during the forward process given $\mX_\textbf{0}$. The model is trained with both L1 and L2 regularization on the adjacency matrix together with the k-hop acyclicity constraint. The model is optimized using the Adam optimizer with increasing weights on the DAG constraint. We will explain this in detail in section \ref{opt}.

\subsubsection{DDCD Nonlinear}

Applying denoising diffusion to non-linear SEMs presents a challenge in that the added noise $\mZ$ cannot be easily separated from $\mX$ in the original feature space. To address this, we adopt a Generalized SEM formulation. Here we assume the observed data $\mX$ is the result of a invertible element-wise nonlinear transformation of latent variable $\mY$, which interact via a linear causal mechanism. 

\begin{equation}
   \label{eqn:non_linear_1}
   \mY = \mY\mW + \mZ_{SEM}, \;\;\; \mX=f_2(\mY) + \mE
\end{equation}
\noindent where $\mZ_{SEM}$ is the noise in the latent causal system on $\mY$, $\mW$ is the weighted adjacency matrix encoding the causal DAG, and $f_2$ is a nonlinear mapping (decoder). To recover this structure, we introduce an inference network $f_1$ (encoder) such that $\mY=f_1(\mX)$. By jointly learning $f_1, f_2$ and $\mW$, we aim to recover the latent causal space $\mY$ where linear dependencies hold. The structural equation becomes:

\begin{equation}
   \label{eqn:non_linear_2}
   \mX = f_2(f_1(\mX)\mW) + \mE\text{, or } \mX = f_2(\mY \mW) + \mE_1
\end{equation} 

This formulation essentially treats the causal discovery problem as learning an autoencoder where the latent space is constrained by a structural prior $\mW$, which is identical to the work proposed in DAG-GNN \citep{yu2019dag}, GAE \citep{ng2019graph}, and recent robust variants for zero-inflated data \citep{zhu2025improved}.

The total training objective combines the nonlinear reconstruction loss, the latent diffusion denoising loss, and the DAG constraints:

\begin{multline}
   \label{eqn:ddcd_nonlinear_obj}
   \min_{W, f_1, f_2} 
\underbrace{\frac{1}{2n}\|X - f_2(f_1(X)W)\|_F^2}_{\text{SEM Reconstruction}} 
+ 
\underbrace{\|(Y_t - Y_t W) - \text{diag}(\sqrt{1-\overline{\alpha}_t})Z(I-W)\|_F^2}_{\text{Latent Denoising}} 
+ \mathcal{R}(W).
\end{multline}
\noindent where $\mathcal{R}(W)$ includes the sparsity and $k$-hop acyclicity constraints. By anchoring the latent representation $Y$ to a denoising objective, we encourage the model to learn a representation where the noise $Z$ is additive and Gaussian, aligning with the identifiability conditions for linear causal models.

This architecture, as shown in Figure \ref{fig:model_architecture}b, is very similar to a Latent Diffusion Model (LDM) \citep{rombach2022high}, where noise is added to the latent representation learned by an autoencoder. %
We can treat $\mY$ as the unobserved latent variable, and the learned nonlinear transformation functions $f_1$ and $f_2$ can be viewed as the encoder and decoder. Furthermore, the learned adjacency matrix $\mW$ for the SEM could be considered as a form of attention or graph neural network.

\subsubsection{DDCD Smooth}

While continuous optimization algorithms excel on synthetic SEMs, recent studies \citep{reisach2021beware,kaiser2021unsuitability} show that their performance often degrades on real-world datasets, where feature scales vary. \citet{reisach2021beware} demonstrated that these methods can exploit ``varsortability'' and infer directionality strictly from increasing variance instead of identifying true causal mechanisms. This is problematic because 1. it violates the equal-variance assumption implicit in standard SEM formulations; 2. real-world data rarely exhibits the same patterns as synthetic data \citep{kaiser2021unsuitability}. To address this limitation, we propose DDCD-Smooth. Instead of learning exact edge weights, DDCD-Smooth learns a normalized adjacency matrix where expected edge weights are scaled to $\frac{1}{d}$. We achieve this by employing an MLP with \textbf{Tanh activation to normalize all features to the range of $[-1, 1]$}. This normalization serves two purposes: (1) it prevents the method from exploiting variance differences for spurious identification, and (2) it makes the method robust to heterogeneous feature scales common in real-world data. An illustration of the architecture of this model is provided in Figure \ref{fig:model_architecture}c. In appendix \ref{ddcd_smooth_app}, we provide a proof showing that with normalized adjacency, we can directly estimate the noise $\mZ$ when the graph is large.

\subsection{Optimization} \label{opt}

In this experiment, the models are optimized using the Adam optimizer \citep{kingma2014adam} since it has fewer restrictions. Distinct from NOTEARS, we replace the dual-ascending augmented Lagrangian formulation with a linear penalty schedule for the acyclicity constraint (See Supplement Section \ref{sec:opt_sup} for our justification and empirical evidence). 

\section{Results} \label{experiments}

\subsection{Denoising Objective Smooths the Optimization Landscape}

\begin{figure*}[t]
\centering
\includegraphics[width=0.90\textwidth]{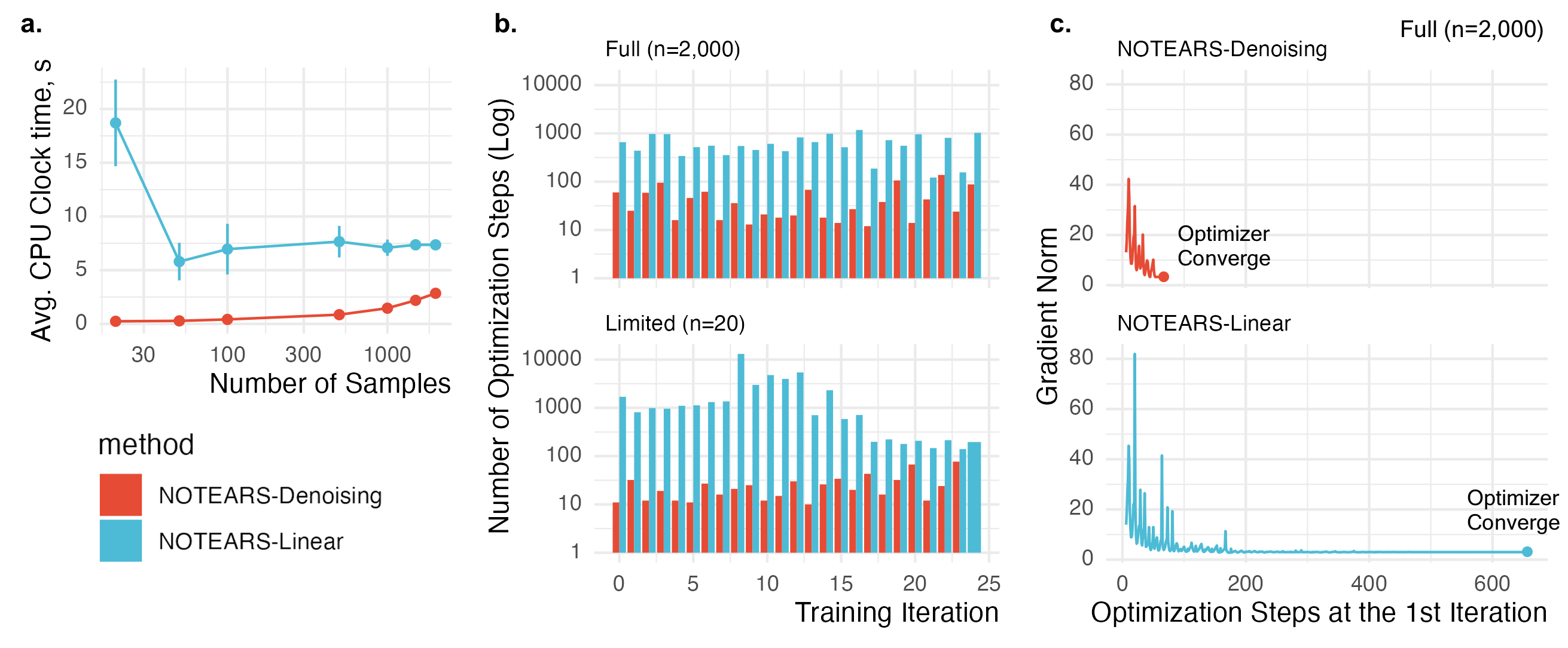}
\caption{\textbf{The denoising objective accelerates convergence by smoothing gradients.} (a) Average runtime vs. sample size (10 runs). (b) L-BFGS-B steps to convergence; NOTEARS-Denoising requires significantly fewer iterations. (c) Gradient norm trajectories during the first iteration on full data ($n=2,000$); the denoising objective yields a stable descent compared to the volatile gradients of the linear baseline. (First 5 steps omitted for clarity).}
\label{fig:notears_results}
\end{figure*}

We first compare NOTEARS-Denoising against NOTEARS-Linear to isolate the impact of the denoising objective on optimization. Except for the noise perturbation process and the denoising objective, both models are identical and are optimized using dual-ascending augmented Lagrangian with the L-BFGS-B optimizer. We evaluate performance on a synthetic Scale-Free graph (20 nodes, degree 10) and analyze gradient behavior and convergence speed.

Note that NOTEARS-Linear \citep{zheng2018dags} converges slowly when the number of samples is very limited (Figure \ref{fig:notears_results}a). This strange behavior occurs because the reconstruction loss landscape contains sharp local minima that produce volatile gradients during optimization. Recent work has also identified instabilities in the acyclicity constraint itself \citep{nazaret2023stable}, though these manifest differently (vanishing constraint gradients vs. volatile loss gradients). Our denoising objective addresses the loss landscape component: it takes many steps for the L-BFGS-B optimizer to converge in each iteration (as shown in Figure \ref{fig:notears_results}b and c). In contrast, the denoising objective inherently smooths these gradients, allowing the optimizer to bypass sharp minima and converge in a fraction of the steps. Additional comparisons are provided in Appendix section \ref{notears_supplement} and Figure \ref{fig:notears_more}.   

\begin{figure}[h]
\includegraphics[width=\textwidth]{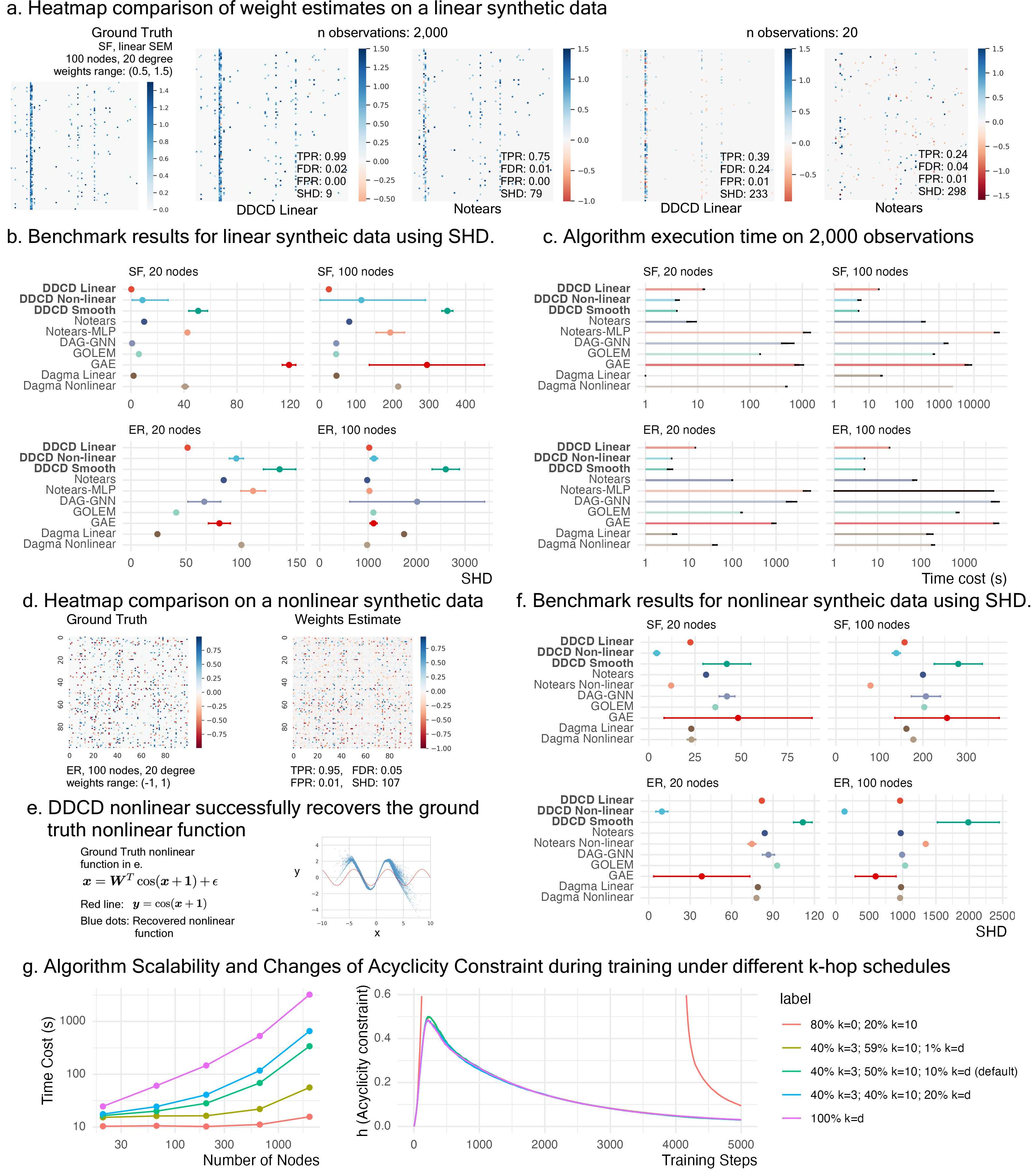}
\caption{DDCD Linear and Nonlinear models demonstrate robust scalability and accuracy on synthetic benchmarks. }
\label{fig:syn_results}
\end{figure}

\subsection{DDCD Demonstrates Robust Scalability and Accuracy on Synthetic Benchmarks}

We evaluated DDCD models across a range of synthetic Scale-Free (SF) and Erd\H{o}s-R\'enyi (ER) graphs (d = 20 - 5,000, degree = 10 - 500) and we compared the performances against a set of continuous optimization baselines, including NOTEARS \citep{zheng2018dags}, NOTEARS-MLP, DAG-GNN \citep{yu2019dag}, GOLEM \citep{ng2020role}, GAE \citep{ng2019graph}, and DAGMA \citep{bello2022dagma}. Baseline methods use default hyperparameters from their original papers. For DDCD, we selected hyperparameters via grid search (Supplementary Figure 8): $T=5000$, linear noise schedule, start noise=$10^{-4}$, end noise=$0.02$. The synthetic data were generated via additive noise models (ANMs) with Gaussian noise, using a linear and seven distinct nonlinear functions, consist of smooth ($\sin, \cos, \tanh$), non-smooth (ReLU), bounded (Sigmoid, $\tanh$), and unbounded (polynomial). Performance was assessed via Structural Hamming Distance (SHD), True Positive Rate (TPR), False Discovery Rate (FDR), False Positive Rate (FPR), and runtime (see Appendix \ref{A:metrics} for metric definitions and a note on identifiability). 

\textbf{Linear Structure Recovery.} On linear benchmarks, the performance of DDCD-Linear (converged at 5,000 steps) matches or exceeds state-of-the-art algorithms such as DAGMA and GOLEM (Figure \ref{fig:syn_results}b) with an  exception on the small ER graphs, while only consuming a fraction of time (Figure \ref{fig:syn_results}c). In terms of runtime, all DDCD models manage to converge within roughly 10 seconds, while other algorithms sometimes take 20 minutes to converge. Visual inspection (Figure \ref{fig:syn_results}a) also confirms that DDCD-Linear yield near perfect structure recovery when n is large (SHD = 9, n=2,000) while recovering most of the structure when n is small (SHD=233, n=20).

\textbf{Nonlinear Disentanglement.} On nonlinear benchmarks, DDCD-Nonlinear (converged at 1,000 steps) shows superior performance on structure recovery (Figure \ref{fig:syn_results}f) and inference speed. Visual inspection (Figure \ref{fig:syn_results}d) shows that it was able to recover most of the structures (SHD=107), even though it is a bit challenging to figure out the signs correctly. One unique feature of DDCD-Nonlinear is that it could learn a precise approximation of the underlying generation mechanism; as shown in (Figure \ref{fig:syn_results}e), the recovered function (blue dots) tightly fits the ground truth cosine transformation ($y=\cos(x+1)$), which further validates the model's capacity for joint structure and function learning.

\textbf{Computational Scalability.} One major advantage of the DDCD framework is its runtime efficiency. When the size of the graph increases from 20 to 100, we see dramatic increases in runtime for standard algorithms (e.g., NOTEARS: 7 seconds to 6 minutes for a SF graph); the runtime of DDCD-Linear increased from 13 seconds to 19 seconds (Figure \ref{fig:syn_results}c). The same patterns persist in ER graphs and under nonlinear mechanisms. Figure \ref{fig:syn_results}g shows the scale of time costs as we increase the size of the graph (the number of nodes increases from 20 to 2,000). With DDCD-Linear + the default k-hop schedule (40\% k=3; 50\% k=10; 10\% Full DAG), it only takes 5.7 minutes to finish 5,000 steps of inference on a graph with 2,000 nodes on a modern GPU (NVIDIA H-100).

\textbf{Adaptive k-hop acyclicity constraint helps resolve DAG violation.} To understand the impact of the k-hop schedule on the training dynamics, we studied the time costs and the changes of acyclicity constraint under a few different schedules and showed the results in Figure \ref{fig:syn_results}g. Applying different k-hop schedules significantly affects algorithm speed. For example, using global DAG checks throughout training (100\% k=d) requires 53.7 minutes for a 2,000 node graph, while our default schedule (40\% k=3; 50\% k=10; 10\% k=d) completes in 5.7 minutes—a nearly 90\% reduction. Importantly, the acyclicity constraint $h$ follows near-identical trajectories in both cases, confirming that the k-hop schedule achieves the same acyclicity guarantees with substantially lower computational cost. Detailed analysis of training dynamics under various k-hop schedules is provided in Supplement \S\ref{khop_schedule_supplement}.

\begin{figure}[t]
\centering
\includegraphics[width=0.87\linewidth]{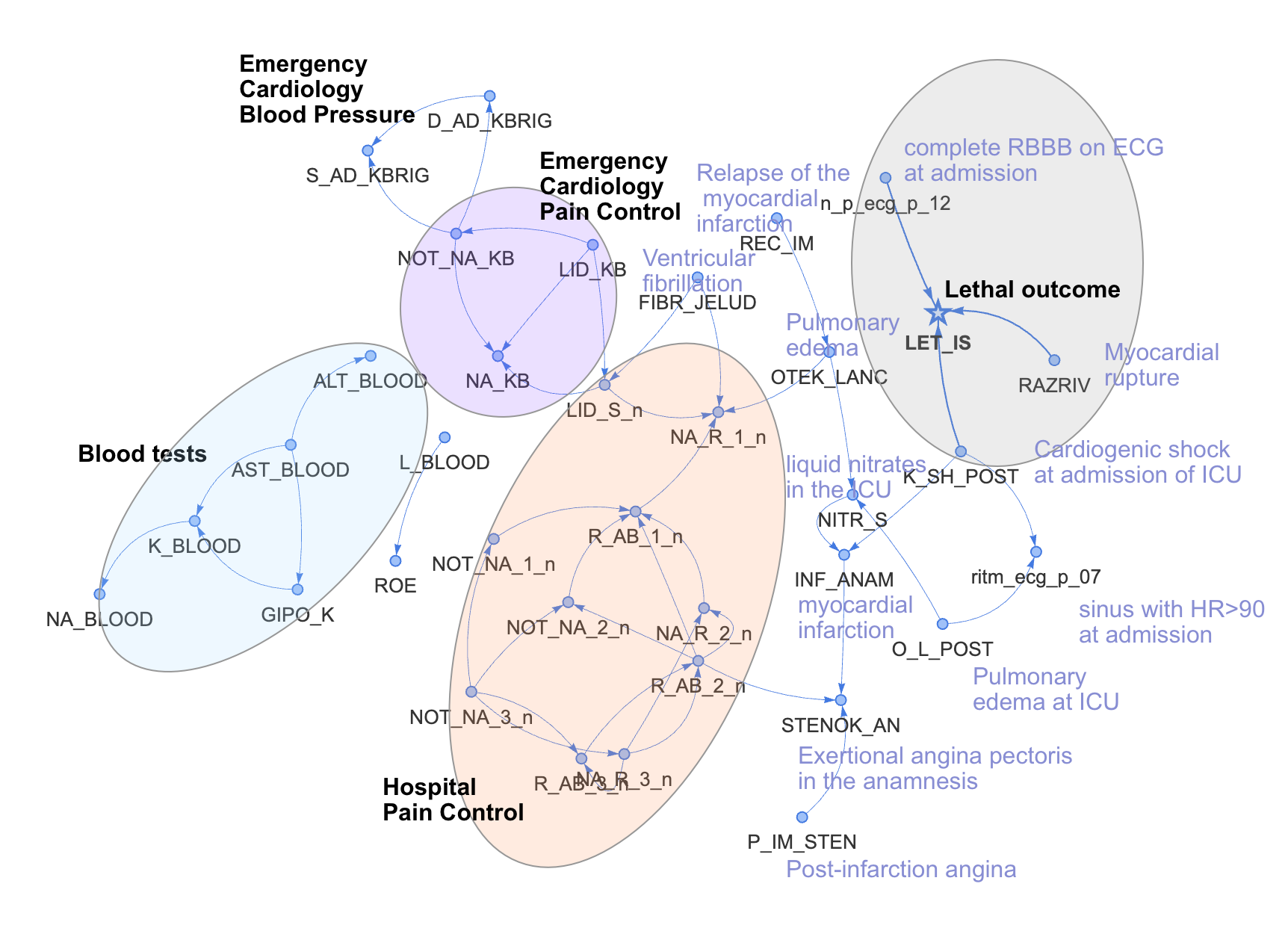}
\caption{Inferred Causal Network around Lethal Outcome in Myocardial Infarction}
\label{fig:mi}
\end{figure}

\subsection{Qualitative Analysis on Real Data}

In this section, we assess the performance of DDCD Smooth on the real-world myocardial infarction dataset. This ran in 9 seconds on an NVIDIA H100 GPU.  Due to the lack of ground truth networks for real data, we visualize regions of the inferred networks and assess quality via domain knowledge.

After training, we extracted all edges with weights above the cut-off threshold (0.2) in the inferred normalized adjacency matrix. This cut-off threshold was set by inspecting the distribution of edge weights and selecting the right tail. As an example, we examine a graph of the 2-hop neighbors around the Lethal Outcome (\texttt{LET\_IS}) node (Figure \ref{fig:mi}).  We identify with ovals several meaningful node clusters in this graph, including lethal outcome with its primary causes, critical conditions and interventions, hospital pain control, emergency cardiology pain control, and blood test results, purely based on the topological relationships among the nodes. The edge weights for a replicate of this network are provided in Appendix \ref{mi_edgeweights_ddcd}. In comparison, as shown in Appendix \ref{mi_edgeweights_notears} and Figure \ref{fig:notears_mi}, networks generated by NOTEARS appear less relevant.

In the network generated by DDCD Smooth, the three most important direct causes of lethal outcome include myocardial rupture, cardiogenic shock, and complete Right Bundle Branch Block (RBBB) on ECG at admission; all of these are known to be conditions with a poor prognosis.  Cardiogenic shock (\texttt{K\_SH\_POST}) is further shown to cause ``sinus ECG rhythm with heart rate $>$ 90" (\texttt{ritm\_ecg\_p\_07}), consistent with a high heart rate (tachycardia) being a symptom of cardiogenic shock \citep{kosaraju}.  Pulmonary edema (\texttt{OTEK\_LANC}) is shown to cause the use of liquid nitrates in the ICU (\texttt{NITR\_S}); this is indeed a common practice for rapidly managing pulmonary edema \citep{purvey}.  Other inferred edges include that NSAID drugs used by the emergency team (\texttt{NOT\_NA\_KB}) cause blood pressure to increase \citep{aljadhey2012comparative}, and that relapsing pain in the 2nd hospitalization period causes NSAID use in the same period.

There are also some node pairs for which plausible edges are inferred, but in an implausible direction. For example, in the lower right of the figure, ``post-infarction angina" (chest pain after the heart attack causing the current hospital admission) is shown to cause ``exertional angina pectoris in the anamnesis" (e.g., a reported history of chest pain after exercise), when the former clearly occurs after the latter.  Still, many of the directed edges appear consistent with known causal relationships.  Another similar analysis on real data on aging  is included in Appendix \ref{aging}.

\section{Discussion and Conclusion} \label{discussion}

In this work, we showed that we could leverage the denoising score-matching objectives of diffusion models to stabilize and accelerate causal structure learning. We demonstrated both theoretically and empirically that the denoising objective acts as a strong regularizer that smooths the loss landscape for faster and better convergence. It allowed the optimizer to avoid sharp local minima, particularly when the number of samples is limited. Furthermore, the introduction of the adaptive k-hop acyclicity curriculum provides a scalable solution to the computationally expensive matrix exponential. In our experiment, it helped reduce the training time by ~90\% while maintaining global acyclicity guarantees. With modern GPU support, the inference on a 2,000 node graph only took 5.7 minutes and the performance is on par with state of the art methods. 

The DDCD-Smooth model also demonstrated strong practical performance on real-world clinical data. By learning a normalized adjacency matrix, the model extracted biologically and medically plausible causal clusters from the myocardial infarction and aging datasets. While there are inherent limitations of inferring directionality from static, cross-sectional data, the clarity of the clusters produced by DDCD-Smooth significantly outperformed the hub-and-spoke artifacts seen in previous methods. Ultimately, DDCD demonstrates that the principles of diffusion modeling can be repurposed to provide a stable, scalable, and explainable path for discovering dependencies in complex tabular data. 

 \newpage

\bibliography{refs}

\newpage
\onecolumn

\appendix
\section{Appendix}

\subsection{Explanation of the k-hop acyclicity constraint} \label{khop_dag_app}

In this section, we explain how to derive the proposed k-hop acyclicity constraint in Equation \ref{eqn:khop_dag} from the NOTEARS DAG constraint in Equation \ref{eqn:notears_dag}. 

The NOTEARS DAG constraint is in the form of $h_{\text{NOTEARS}}(\mW) = \text{tr}(e^{\mW \circ \mW})-d$, where $\circ$ is the Hadamard product, $e^{\mW}$ is the matrix exponential of $\mW$, and $\text{tr}()$ is the trace of a matrix. In this case, matrix exponential is the sum of a weighted power series as shown below.

\begin{equation}
e^{\mW} = \sum_{j=0}^{\infty} \frac{1}{j!} \mW^j.
\label{eqn:matrix_exp}
\end{equation}

The trace of the summed matrix and the sum of all the traces are equivalent. At the same time, since $\mW^0$ is simply the identity matrix, whose trace equals to the value of $d$, we can rewrite the NOTEARS DAG function in the following form: 

\begin{equation}
h_{\text{NOTEARS}}(\mW) = \sum_{j=0}^{\infty} \frac{1}{j!} \text{tr}((\mW \circ \mW)^j) - d = \sum_{j=1}^{\infty} \frac{1}{j!} \text{tr}((\mW \circ \mW)^j) .
\label{eqn:notears_full}
\end{equation}

To account for cycles within $k$ hops, we compute: 

\begin{equation}
h_{\text{k-hop}}(\mW, k) = \sum_{j=1}^{k+1} \frac{1}{j!} \text{tr}((\mW \circ \mW)^j)
\end{equation}

As mentioned in the main text, in the case when values in the weighted adjacency matrix are tiny, it might be helpful to multiply the values in the adjacency matrix by a constant multiplier $\gamma$ and then remove it after the trace calculation.  Then the equation can be expressed in the following form: 

\begin{equation}
   h_{\text{k-hop}}(\mW, k, \gamma) = \sum_{j=1}^{k+1}\frac{1}{j!\gamma^{2j}}\text{tr}((\gamma\mW \circ \gamma\mW)^j)
\end{equation}

If we keep a running product for $j!$, $\gamma^{2j}$, and $(\gamma\mW \circ \gamma\mW)^j$, we can keep the complexity within $\mathcal{O}(d^2)$. 

\subsection{Special Case in DDCD Smooth} \label{ddcd_smooth_app}

In DDCD Smooth, all the inputs are transformed into the range of -1 to 1 through MLP and the Tanh activation function.  We expected to learn a normalized adjacency matrix $\hat{\mW}$, where the expected value is $\frac{1}{d}$. This normalized adjacency matrix would be conceptually similar to the normalized adjacency matrix in graph convolution \citep{kipf2016semi}. Under these assumptions, Theorem 2 shows that we can  predict the added noise $\mZ$ directly.

\begin{theorem}
   With a normalized adjacency matrix, we can directly infer the added noise $\mZ$.
\end{theorem}

\begin{proof}
   Starting from Equation \ref{eqn:diffusion_forward_matrix_w}, let's pick a random sample $\vx$ and perturb that with a Gaussian noise vector $\vz \in \mathcal{N}(0, 1)$ to build $\vx_t$. 

   \begin{equation}
   \label{eqn:diffusion_forward_matrix_w_sample}
   \mW^T\vx_t = \sqrt{\overline{\alpha_t}}\mW^T\vx_0 + \sqrt{1-\overline{\alpha_t}}\mW^T\vz,
    \end{equation}

We can in fact write each element in $\mW^T\vz$ as a form of weighted Gaussian mixtures. Since all values in $\vz$ are standard Gaussian noise with a mean of 0 and variance of 1, the weighted sum of such a mixture will also be centered at 0.  Given that the expected value of edge weight in $\mW$ is $\frac{1}{d-1}$ and there are $d-1$ entries, the expected value for the entire variance is $\sum_{i=0}^d \frac{1-\overline{\alpha_t}}{d^2} = \frac{1-\overline{\alpha_t}}{d}$. When $d$ is large, this variance of $\mW^T\vx_t$ will be much smaller than the variance term in $\vx_t$, which is $1-\overline{\alpha_t}$. When $d$ is really large and the diffusion coeffient is small, we can therefore use $\mW^T\vx_t$ to approximate the unperturbed $\vx_0$. This argument is very similar to the Central Limit Theorem, but on noise. As a result,  $\mW^T\vx_t - \vx_t$ will give us a close estimate of the added noise $z$. 

\end{proof}

\subsection{Performance on Large Graphs}

Here, we include two sample weight estimates on larger graphs with 1,000 nodes. The main structures of the graphs are recovered (Figure ~\ref{fig:large}). 

\begin{figure}[ht]
\centering
\includegraphics[width=4in]{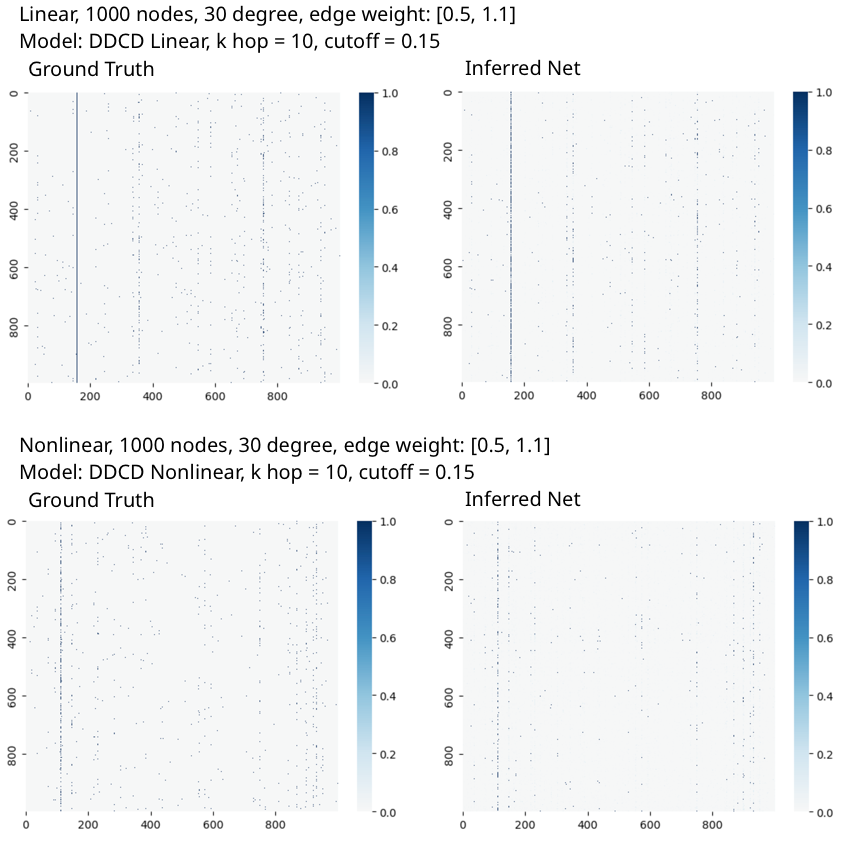}
\caption{Example Weight Estimates on Graphs with 1,000 nodes. Number of samples in both cases is 2,000.}
\label{fig:large}
\end{figure}

\subsection{Metrics}  
\label{A:metrics}

Since the inferred graphs are directed graphs, we use the same evaluation methods used in NOTEARS. Since we are not generating non-directed edge predictions at all, here is a simplified description of the metrics that we are using. 

\begin{enumerate}

\item True Positive Rate (TPR) is defined as 
\begin{equation}
\text{TPR} = \frac{\text{True Positive}} {\text{Condition Positive}}
\end{equation}
True positive is the number of cases when the predicted association exists in the condition in the correct direction. Condition positive is the total number of true edges in the ground truth graph.  

\item False Discovery Rate (FDR) is defined as
\begin{equation}
\text{FDR} = \frac{\text{False Positive} + \text{Reverse}} {\text{Prediction Positive}}
\end{equation}
False positive is the number of cases when the predicted association does not exist in the condition. Reverse is the number of cases when the predicted association exists in the condition but in the opposite direction. Prediction Positive is the total number of positive predictions in the inferred graph.  

\item False Positive Rate (FPR) is defined as
\begin{equation}
\text{FPR} = \frac{\text{False Positive} + \text{Reverse}} {\text{Condition Negative}}
\end{equation} 
Condition negative is the total number of edges that do not exist.   

\item Structural Hamming Distance (SHD) is a measure used to quantify the difference between two directed acyclic graphs (DAGs). It counts the number of operations, including adding an edge, removing an edge, and reversing an edge, required to transform one graph into another. Here, the SHD is the sum of reversed positive predictions, false positive predictions regardless of direction, and false negative predictions regardless of direction.  

\end{enumerate}

\textbf{Note on Identifiability and SHD:} We report SHD against the ground-truth DAG. For linear models with Gaussian noise, the causal structure is generally identifiable only up to the Markov Equivalence Class (MEC), meaning some edge reversals may not be true errors. Following standard practice in continuous DAG learning \citep{zheng2018dags,ng2020role,bello2022dagma}, we report SHD against the ground-truth DAG to enable direct comparison with prior work. For nonlinear models with additive noise, identifiability of the full DAG follows from ANM assumptions \citep{peters2014causal}, even with Gaussian noise, because the nonlinearity breaks the symmetry that causes equivalence classes in the linear case.

\subsection{Inferred examples from NOTEARS-Denoising} \label{notears_supplement}

\begin{figure}[htb]
\includegraphics[width=\textwidth]{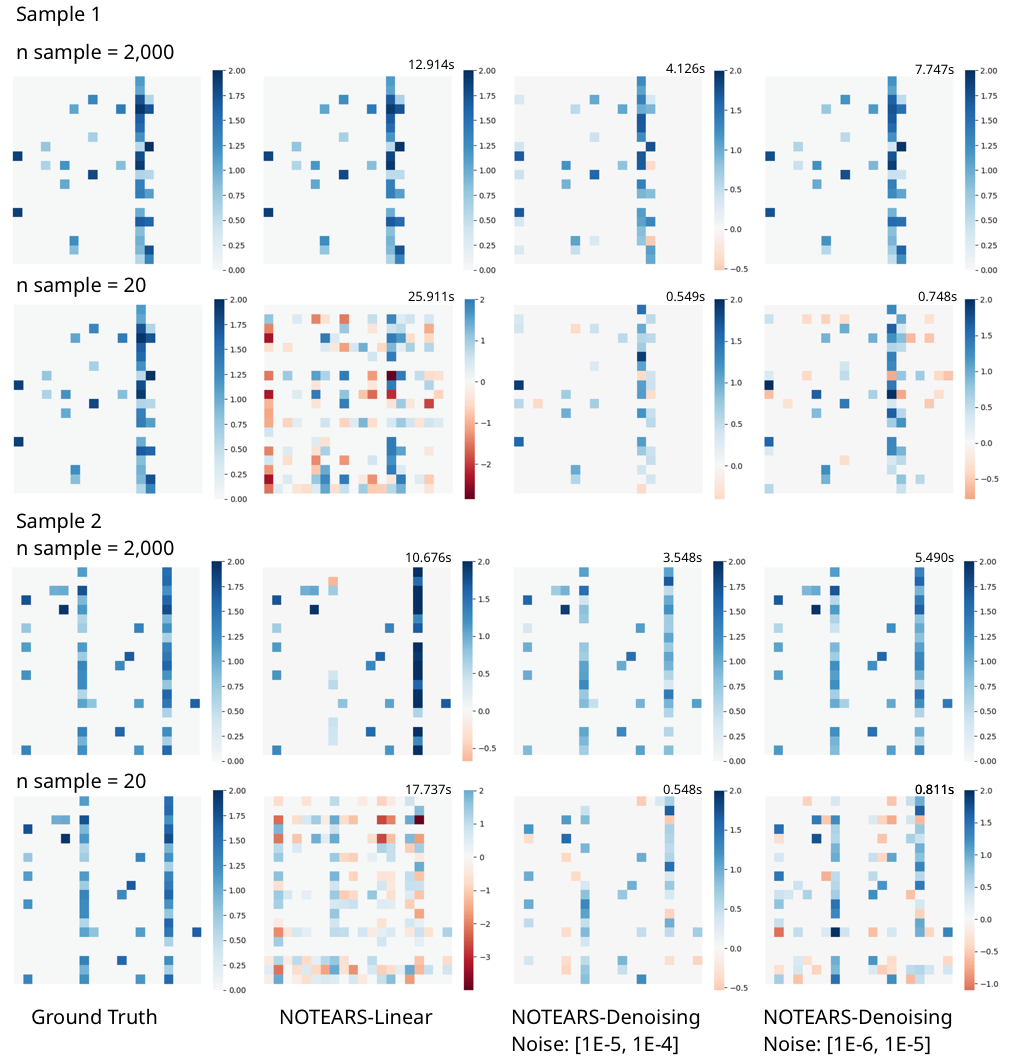}
\caption{Comparison of some results from NOTEARS-Linear and NOTEARS-Denoising. Execution times are displayed on the top-right corner of each figure. }
\label{fig:notears_more}
\end{figure}

In Figure \ref{fig:notears_more} we show some additional comparisons between results from NOTEARS-Linear and NOTEARS-Denoising. Overall, as reported in the main paper, when the number of samples is limited (2nd and 4th rows), the results from NOTEARS-Linear (2nd column) may include a lot of noise. This could be resolved by using the denoising diffusion objective (3rd and 4th columns). When the number of samples is sufficient (1st and 3rd rows), in most cases, NOTEARS-Linear will generate very good inference but in some cases, such as Sample 2 in rows 3 and 4, it may end up in a local minimum. On the other hand, using the denoising objective does come with a cost. When the added noise is not small enough, it may introduce some small noisy values in the inferred matrix. This could be resolved by adding smaller noises instead (column 4) but smaller noises will also increase the runtime.

\subsection{Linear multiplier outperforms Augmented Lagrangian in optimization} \label{sec:opt_sup}

Another experiment design we tested involves replacing the dual-ascending augmented Lagrangian optimization with a simple linear multiplier using training epoch steps. Here We justify this simplification for the following reasons, along with empirical evidence: \textbf{1. Smoother Optimization and Better Convergence:} As shown in Figure \ref{fig:notears_results}, the denoising objective effectively smooths the loss gradient. This stability allows the optimizer to converge effectively using a monotonic linear schedule. \textbf{2. Bounded Edge Weights:} In models like DDCD-Smooth, feature normalization restricts edge weights to a small range (typically $|w_{ij}| < 1$). The values of the quadratic penalty term of the augmented Lagrangian $\rho\|h(\mW)\|^2$ become extremely small and cannot be used as a penalty. 

In our empirical experiment, we found that while the Augmented Lagrangian approach maintains a very low FDR and FPR, its TPR degrades significantly as the graph node count increases. In contrast, the linear multiplier maintains high discovery power with a TPR above 0.95 and maintains a substantially lower Structural Hamming Distance (SHD) across all scales. Furthermore, the linear approach offers a more efficient optimization path, consistently reducing total computation time. This experiment clearly shows that at least for our diffusion setup, linear multiplier is a better choice compared with the Augmented Lagrangian. 

\begin{table}[h]
\centering
\caption{Optimization Strategy Comparison: Augmented Lagrangian vs. Linear Multiplier} \label{tab:opt}
\begin{tabular}{lrlclllll}
\toprule
Optimizer & Nodes & SHD & TPR & FPR & FDR & Time(s) \\
\midrule
Augmented Lagrangian & 20  & 32.5±3.0   & 0.45±0.03 & \textbf{0.01±0.01} & \textbf{0.04±0.02} & 15.8±8.2 \\
                     & 50  & 135.8±11.7 & 0.23±0.04 & \textbf{0.00±0.00} & \textbf{0.10±0.08} & 13.8±0.8 \\
                     & 100 & 328.5±18.1 & 0.13±0.01 & \textbf{0.00±0.00} & \textbf{0.15±0.03} & 27.4±8.1 \\
\midrule
Linear Multiplier    & 20  & \textbf{10.7±5.9}  & \textbf{0.96±0.04} & 0.07±0.04 & 0.14±0.07 & \textbf{8.9±1.7} \\
                     & 50  & \textbf{46.2±12.5} & \textbf{0.97±0.01} & 0.04±0.01 & 0.21±0.05 & \textbf{12.8±2.1} \\
                     & 100 & \textbf{139.3±16.2}& \textbf{0.96±0.01} & 0.03±0.00 & 0.26±0.03 & \textbf{15.9±3.3} \\
\bottomrule
\end{tabular}
\end{table}

\subsection{Detailed analysis of the k-hop schedule Training Dynamics} \label{khop_schedule_supplement}

\begin{figure}[htb]
\includegraphics[width=\textwidth]{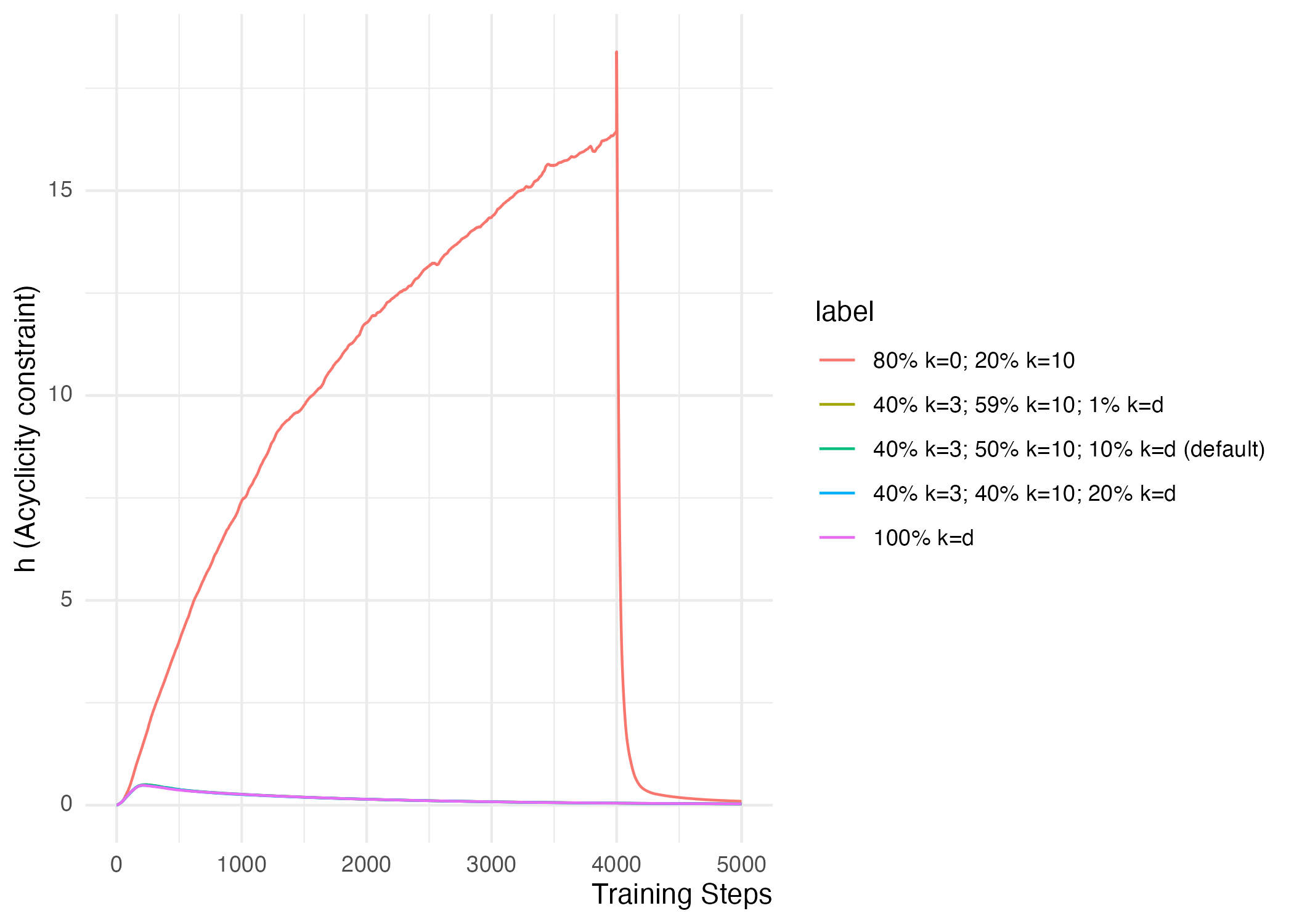}
\caption{Change of DAG constraint h under different k-hop schedules}
\label{fig:k}
\end{figure}

In this section, we provide detailed analysis of how different k-hop schedules affect training dynamics and the acyclicity constraint. This experiment was done on a 100-node SF graph with linear synthetic data, and we used DDCD-Linear to illustrate the point. As shown in the figure, applying different k-hop schedules can significantly affect the speed of the algorithm. For example, in the case where we applied global DAG check all the time (100\% k=d), the entire training process for a 2,000 node graph would take 53.7 minutes while if we used the default schedule (40\% k=3; 50\% k=10; 10\% k=d), we could reduce the time to 5.7 minutes, nearly 90\% reduction. At the same time, the movements of acyclicity constraint $h$ during training are near identical in these two cases. 

We also conducted an experiment where we applied no DAG constraint at the first 80\% of training steps and $k=10$ for the last 20\%, the acyclicity score $h$ skyrocketed at the beginning but dropped quickly after we put on the constraint (after step 4,000) (Figure \ref{fig:k}). In the cases where we did apply some acyclicity constraint at the beginning, the computed $h$ still rapidly increased from 0.0 to 0.5 in all cases, because at this stage, the SEM reconstruction loss is much larger than the acyclicity constrain and is the main factor that drives the direction of the optimizer. After this initial stage, the training starts to balance the power of SEM reconstruction and DAG constraint. Based on these observations, we concluded that 1. the sharp increase of $h$ could be prevented by applying a local acyclicity constraint; 2. spending resources on computing the exact value of global DAG constraint is not necessary for early stage training; 3. using a k-hop acyclicity schedule is as effective as full DAG calculation while consuming a fraction of time.

\subsection{Hyperparameter Search}

To study the sensitivity of DDCD to diffusion noise parameters, we conducted a comprehensive grid search. All experiments were performed on synthetic SF graphs with 50 nodes, degree 10, and 2,000 samples using linear SEMs with Gaussian noise. 

Here we explored:
\begin{itemize}
    \item Noise Schedule: Linear, Cosine, Power
    \item Diffusion timesteps $T$: 1000, 2500, 5000, 10000
    \item Start noise: 1e-5, 1e-4, 1e-3
    \item End noise: 0.01, 0.02, 0.05, 0.1
\end{itemize}

Based on our experiment, all three schedules achieved comparable SHD performances as shown in Figure \ref{fig:hyper}a. The cosine schedule seems to yield overall better performance, but it's not statistically significant. We also observe similar performances across Diffusion Steps $T$, suggesting that DDCD is robust across noise schedules and diffusion steps. 

We also did hyperparameter search in terms of the start noise and end noise and the results are included in Figure \ref{fig:hyper}b. The best performing combination happens when Start Noise is 1e-4 or 1e-5 and End Noise is 1e-2. Note that, in our other experiments, we followed standard Diffusion literature and use $T=5,000$, Start Noise = 1e-4 and End Noise = 0.02. This is not the optim setup but yield similar performance.  

\begin{figure}[t]
\centering
\includegraphics[width=\linewidth]{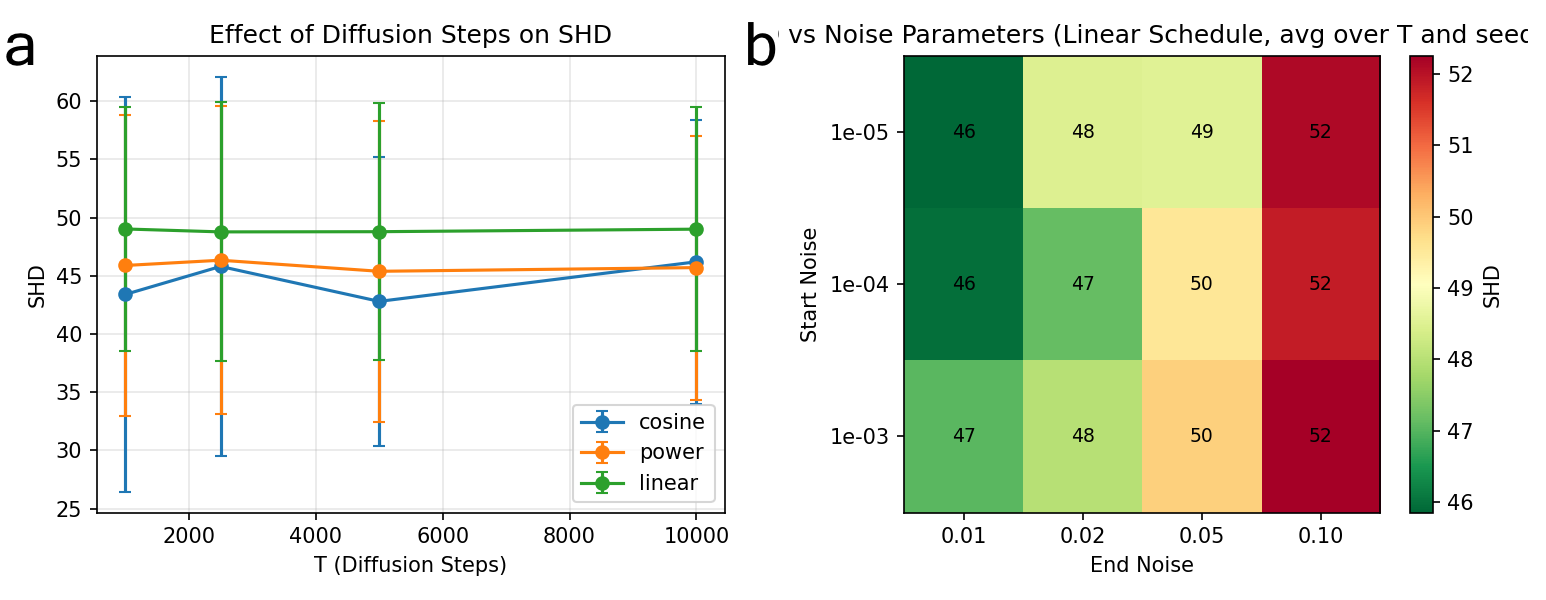}
\caption{Effect of Diffusion timestep and noise schedule}
\label{fig:hyper}
\end{figure}

\subsection{Directionality Analysis}

We investigated whether the gradient smoothing provided by the denoising objective affects directionality. To isolate this effect, we decomposed the SHD metric into false positives, false negatives, and edge reversals. In both low-sample and high-sample conditions, DDCD persistently reduces directional errors compared to NOTEARS (Figure \ref{fig:reversed}). These findings indicate that the denoising objective is, in fact, helpful for the correctness of directionality. 

\begin{figure}[t]
\centering
\includegraphics[width=\linewidth]{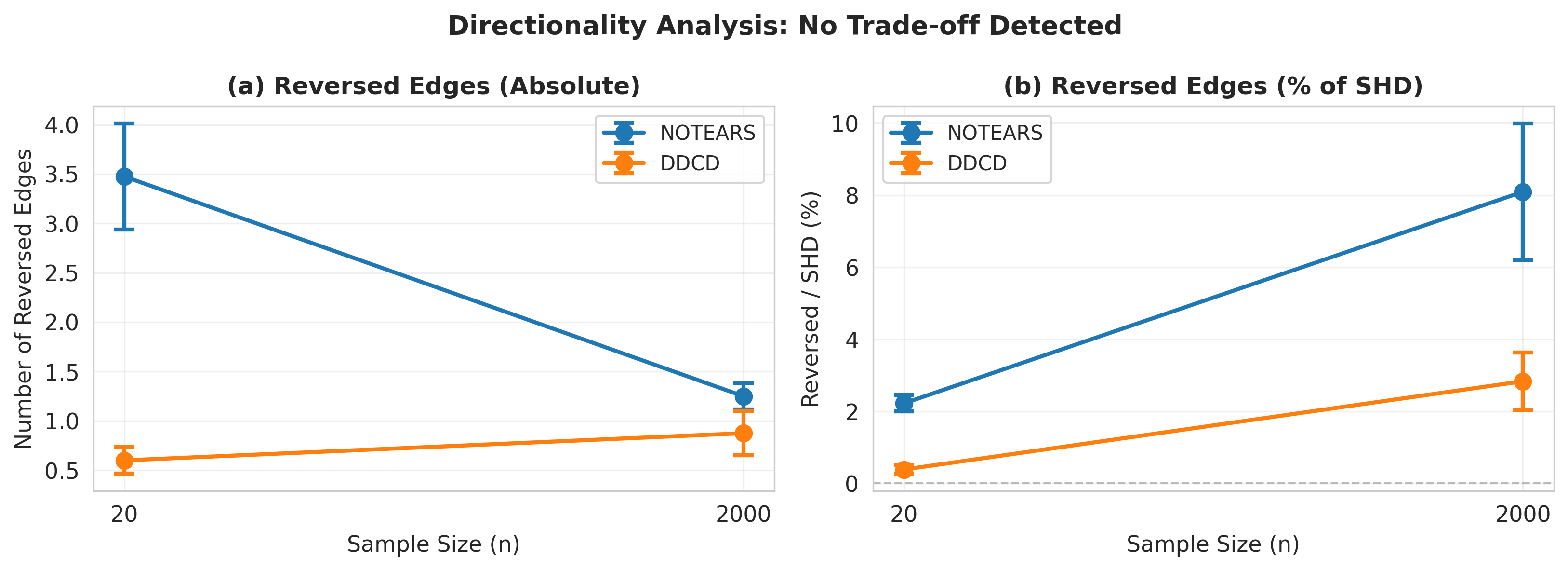}
\caption{Reversed Edge Analysis}
\label{fig:reversed}
\end{figure}

\subsection{Full results for main benchmark experiments (Figure 3)}
\fontsize{8}{11}\selectfont

\begin{longtable}{rrlllllll}
\caption{Main benchmark results (SHD, TPR, FDR, FPR across all methods)}\\
\hline
N & Data & Graph & Method & SHD & TPR & FDR & FPR & Time(s) \\
\hline
\endfirsthead
\hline
N & Data & Graph & Method & SHD & TPR & FDR & FPR & Time(s) \\
\hline
\endhead
\hline
\endfoot

2000 & linear & SF-20 & DAG-GNN & 0.9±1.0 & 0.99±0.01 & 0.02±0.02 & 0.01±0.01 & 548.8 \\
2000 & linear & SF-20 & DDCD Linear & \textbf{0.3±0.9} & \textbf{1.00±0.01} & 0.01±0.02 & \textbf{0.00±0.01} & 13.1 \\
2000 & linear & SF-20 & DDCD Non-linear & 9.0±18.3 & 0.87±0.25 & 0.07±0.19 & 0.02±0.04 & 4.1 \\
2000 & linear & SF-20 & DDCD Smooth & 50.7±7.0 & 0.17±0.16 & 0.73±0.18 & 0.17±0.06 & 4.0 \\
2000 & linear & SF-20 & GAE & 139.3±3.2 & 0.60±0.03 & 0.86±0.00 & 1.43±0.10 & 890.5 \\
2000 & linear & SF-20 & GOLEM & 6.0±0.0 & 0.92±0.00 & 0.08±0.00 & 0.03±0.00 & 156.4 \\
2000 & linear & SF-20 & Notears & 7.0±0.0 & 0.88±0.00 & 0.06±0.00 & 0.02±0.00 & 7.8 \\
2000 & linear & SF-20 & Notears Non-linear & 42.6±1.1 & 0.48±0.01 & 0.51±0.01 & 0.19±0.01 & 1247.4 \\
2000 & linear & SF-20 & PC & 27.0±0.0 & 0.60±0.00 & 0.21±0.00 & 0.06±0.00 & 4.6 \\
2000 & linear & SF-20 & DAGMA Linear & 2.0±0.0 & 0.96±0.00 & \textbf{0.00±0.00} & \textbf{0.00±0.00} & \textbf{1.0} \\
2000 & linear & SF-20 & DAGMA Non-linear & 40.1±2.1 & 0.37±0.04 & 0.37±0.04 & 0.08±0.01 & 581.9 \\
2000 & linear & SF-20 & GES & 74.0±0.0 & 0.73±0.00 & 0.65±0.00 & 0.51±0.00 & 19.9 \\
\addlinespace
2000 & linear & SF-100 & DAG-GNN & 45.4±4.4 & 0.91±0.01 & 0.05±0.01 & \textbf{0.00±0.00} & 1607.6 \\
2000 & linear & SF-100 & DDCD Linear & \textbf{25.1±4.3} & \textbf{0.98±0.00} & 0.07±0.01 & \textbf{0.00±0.00} & 19.0 \\
2000 & linear & SF-100 & DDCD Non-linear & 101.7±122.9 & 0.77±0.24 & 0.13±0.21 & 0.01±0.01 & \textbf{5.3} \\
2000 & linear & SF-100 & DDCD Smooth & 349.8±15.0 & 0.05±0.03 & 0.87±0.06 & 0.02±0.00 & 5.0 \\
2000 & linear & SF-100 & GAE & 527.4±535.4 & 0.45±0.30 & 0.49±0.29 & 0.09±0.14 & 7076.6 \\
2000 & linear & SF-100 & GOLEM & 45.0±0.0 & 0.87±0.00 & \textbf{0.01±0.00} & \textbf{0.00±0.00} & 714.2 \\
2000 & linear & SF-100 & Notears & 71.0±0.0 & 0.81±0.00 & 0.04±0.00 & \textbf{0.00±0.00} & 360.7 \\
2000 & linear & SF-100 & Notears Non-linear & 193.3±38.8 & 0.58±0.11 & 0.31±0.06 & 0.02±0.00 & 45896.1 \\
2000 & linear & SF-100 & PC & 294.0±0.0 & 0.28±0.00 & 0.52±0.00 & 0.02±0.00 & 52.8 \\
2000 & linear & SF-100 & DAGMA Linear & 46.0±0.0 & 0.86±0.00 & \textbf{0.01±0.00} & \textbf{0.00±0.00} & 22.6 \\
2000 & linear & SF-100 & DAGMA Non-linear & 213.6±5.0 & 0.34±0.01 & 0.08±0.02 & \textbf{0.00±0.00} & 2385.2 \\
2000 & linear & SF-100 & GES & 98.0±0.0 & 0.94±0.00 & 0.24±0.00 & 0.02±0.00 & 14992.6 \\
\addlinespace
2000 & linear & ER-20 & DAG-GNN & 66.4±14.9 & 0.58±0.21 & 0.31±0.05 & 0.29±0.10 & 2360.7 \\
2000 & linear & ER-20 & DDCD Linear & 51.3±1.1 & \textbf{0.91±0.01} & 0.32±0.00 & 0.48±0.01 & 14.0 \\
2000 & linear & ER-20 & DDCD Non-linear & 99.2±6.7 & 0.35±0.10 & 0.51±0.06 & 0.38±0.07 & 4.0 \\
2000 & linear & ER-20 & DDCD Smooth & 134.6±14.6 & 0.42±0.19 & 0.75±0.08 & 1.30±0.48 & \textbf{3.7} \\
2000 & linear & ER-20 & GAE & 79.5±11.0 & 0.39±0.14 & 0.34±0.09 & 0.22±0.10 & 898.7 \\
2000 & linear & ER-20 & GOLEM & 41.0±0.0 & 0.71±0.00 & 0.18±0.00 & 0.18±0.00 & 161.6 \\
2000 & linear & ER-20 & Notears & 83.0±0.0 & 0.42±0.00 & 0.39±0.00 & 0.30±0.00 & 98.7 \\
2000 & linear & ER-20 & Notears Non-linear & 110.6±11.1 & 0.12±0.10 & 0.66±0.25 & 0.41±0.36 & 5204.3 \\
2000 & linear & ER-20 & PC & 104.0±0.0 & 0.09±0.00 & 0.80±0.00 & 0.39±0.00 & 1.7 \\
2000 & linear & ER-20 & DAGMA Linear & \textbf{24.0±0.0} & 0.86±0.00 & 0.13±0.00 & 0.14±0.00 & 4.7 \\
2000 & linear & ER-20 & DAGMA Non-linear & 100.0±0.0 & 0.00±0.00 & \textbf{0.00±0.00} & \textbf{0.00±0.00} & 39.7 \\
2000 & linear & ER-20 & GES & 120.0±0.0 & 0.57±0.00 & 0.66±0.00 & 1.21±0.00 & 187.8 \\
\addlinespace
2000 & linear & ER-100 & DAG-GNN & 2012.4±1391.4 & 0.26±0.33 & 0.81±0.07 & 0.45±0.70 & 5350.8 \\
2000 & linear & ER-100 & DDCD Linear & 1023.5±2.3 & 0.02±0.00 & 0.75±0.01 & 0.02±0.00 & 19.0 \\
2000 & linear & ER-100 & DDCD Non-linear & 1058.0±46.4 & 0.03±0.02 & 0.71±0.25 & 0.03±0.02 & \textbf{5.0} \\
2000 & linear & ER-100 & DDCD Smooth & 2603.0±281.4 & 0.30±0.06 & 0.90±0.01 & 0.68±0.14 & \textbf{5.0} \\
2000 & linear & ER-100 & GAE & 1189.7±131.1 & 0.07±0.04 & 0.79±0.03 & 0.07±0.04 & 5537.6 \\
2000 & linear & ER-100 & GOLEM & 1109.0±0.0 & 0.03±0.00 & 0.87±0.00 & 0.04±0.00 & 707.5 \\
2000 & linear & ER-100 & Notears & 983.0±0.0 & 0.01±0.00 & 0.57±0.00 & \textbf{0.00±0.00} & 72.6 \\
2000 & linear & ER-100 & Notears Non-linear & 1027.2±37.2 & 0.00±0.00 & 0.92±0.06 & 0.02±0.01 & 2356.0 \\
2000 & linear & ER-100 & PC & 1094.0±0.0 & 0.06±0.00 & 0.80±0.00 & 0.06±0.00 & 112.6 \\
2000 & linear & ER-100 & DAGMA Linear & 1742.6±0.5 & \textbf{0.34±0.00} & 0.77±0.00 & 0.28±0.00 & 163.9 \\
2000 & linear & ER-100 & DAGMA Non-linear & \textbf{980.0±0.0} & 0.00±0.00 & \textbf{0.00±0.00} & \textbf{0.00±0.00} & 191.7 \\
2000 & linear & ER-100 & GES & - & - & - & - & >12hr \\
\addlinespace
2000 & Non-linear 1 & SF-20 & DAG-GNN & 42.3±4.2 & 0.25±0.07 & 0.19±0.12 & 0.02±0.02 & 507.1 \\
2000 & Non-linear 1 & SF-20 & DDCD Linear & 22.5±1.1 & 0.59±0.02 & 0.06±0.00 & 0.01±0.00 & 14.0 \\
2000 & Non-linear 1 & SF-20 & DDCD Non-linear & \textbf{8.4±1.0} & \textbf{0.84±0.02} & \textbf{0.00±0.00} & \textbf{0.00±0.00} & 4.0 \\
2000 & Non-linear 1 & SF-20 & DDCD Smooth & 42.2±12.9 & 0.33±0.23 & 0.60±0.25 & 0.18±0.09 & 4.0 \\
2000 & Non-linear 1 & SF-20 & GAE & 90.4±44.9 & 0.78±0.35 & 0.68±0.23 & 0.76±0.48 & 987.5 \\
2000 & Non-linear 1 & SF-20 & GOLEM & 36.0±0.0 & 0.46±0.00 & 0.27±0.00 & 0.07±0.00 & 169.5 \\
2000 & Non-linear 1 & SF-20 & Notears & 21.0±0.0 & 0.62±0.00 & 0.06±0.00 & 0.01±0.00 & 2.0 \\
2000 & Non-linear 1 & SF-20 & Notears Non-linear & 12.1±0.3 & 0.79±0.01 & 0.03±0.01 & 0.01±0.00 & 93.9 \\
2000 & Non-linear 1 & SF-20 & PC & 27.0±0.0 & 0.58±0.00 & 0.25±0.00 & 0.07±0.00 & 2.1 \\
2000 & Non-linear 1 & SF-20 & DAGMA Linear & 23.0±0.0 & 0.58±0.00 & 0.06±0.00 & 0.01±0.00 & \textbf{0.0} \\
2000 & Non-linear 1 & SF-20 & DAGMA Non-linear & 23.2±2.1 & 0.55±0.04 & 0.02±0.02 & 0.01±0.00 & 360.5 \\
2000 & Non-linear 1 & SF-20 & GES & 48.0±0.0 & 0.54±0.00 & 0.59±0.00 & 0.29±0.00 & 11.1 \\
\addlinespace
2000 & Non-linear 1 & SF-100 & DAG-GNN & 206.8±33.3 & 0.38±0.11 & 0.08±0.03 & \textbf{0.00±0.00} & 1784.4 \\
2000 & Non-linear 1 & SF-100 & DDCD Linear & 157.8±3.1 & 0.54±0.01 & 0.08±0.00 & \textbf{0.00±0.00} & 18.5 \\
2000 & Non-linear 1 & SF-100 & DDCD Non-linear & 153.0±12.1 & 0.56±0.04 & 0.09±0.02 & \textbf{0.00±0.00} & \textbf{5.0} \\
2000 & Non-linear 1 & SF-100 & DDCD Smooth & 281.1±55.1 & 0.38±0.27 & 0.58±0.21 & 0.03±0.01 & \textbf{5.0} \\
2000 & Non-linear 1 & SF-100 & GAE & 275.8±194.1 & 0.44±0.33 & 0.31±0.28 & 0.02±0.03 & 3872.4 \\
2000 & Non-linear 1 & SF-100 & GOLEM & 203.0±0.0 & 0.40±0.00 & 0.09±0.00 & \textbf{0.00±0.00} & 680.6 \\
2000 & Non-linear 1 & SF-100 & Notears & 154.0±0.0 & 0.55±0.00 & 0.07±0.00 & \textbf{0.00±0.00} & 42.3 \\
2000 & Non-linear 1 & SF-100 & Notears Non-linear & \textbf{79.5±5.1} & \textbf{0.83±0.01} & 0.10±0.01 & 0.01±0.00 & 11903.1 \\
2000 & Non-linear 1 & SF-100 & DAGMA Linear & 162.0±0.0 & 0.52±0.00 & 0.06±0.00 & \textbf{0.00±0.00} & 11.0 \\
2000 & Non-linear 1 & SF-100 & DAGMA Non-linear & 176.7±3.4 & 0.45±0.01 & \textbf{0.02±0.00} & \textbf{0.00±0.00} & 1263.5 \\
2000 & Non-linear 1 & SF-100 & GES & 316.0±0.0 & 0.50±0.00 & 0.58±0.00 & 0.05±0.00 & 21959.0 \\
\addlinespace
2000 & Non-linear 1 & ER-20 & DAG-GNN & 86.7±4.3 & 0.12±0.05 & 0.22±0.11 & 0.04±0.02 & 885.5 \\
2000 & Non-linear 1 & ER-20 & DDCD Linear & 81.9±1.2 & 0.18±0.01 & 0.23±0.05 & 0.05±0.02 & 14.0 \\
2000 & Non-linear 1 & ER-20 & DDCD Non-linear & \textbf{23.1±4.3} & 0.76±0.05 & \textbf{0.00±0.01} & \textbf{0.00±0.01} & \textbf{4.0} \\
2000 & Non-linear 1 & ER-20 & DDCD Smooth & 111.6±6.7 & 0.30±0.06 & 0.71±0.05 & 0.75±0.15 & \textbf{4.0} \\
2000 & Non-linear 1 & ER-20 & GAE & 34.5±32.2 & \textbf{0.86±0.27} & 0.24±0.21 & 0.25±0.15 & 1019.8 \\
2000 & Non-linear 1 & ER-20 & GOLEM & 93.0±0.0 & 0.12±0.00 & 0.48±0.00 & 0.12±0.00 & 292.4 \\
2000 & Non-linear 1 & ER-20 & Notears & 79.0±0.0 & 0.25±0.00 & 0.31±0.00 & 0.12±0.00 & 3.0 \\
2000 & Non-linear 1 & ER-20 & Notears Non-linear & 74.7±2.8 & 0.45±0.03 & 0.41±0.03 & 0.31±0.02 & 415.7 \\
2000 & Non-linear 1 & ER-20 & PC & 105.0±0.0 & 0.19±0.00 & 0.71±0.00 & 0.48±0.00 & 6.1 \\
2000 & Non-linear 1 & ER-20 & DAGMA Linear & 79.0±0.0 & 0.20±0.00 & 0.17±0.00 & 0.04±0.00 & \textbf{0.0} \\
2000 & Non-linear 1 & ER-20 & DAGMA Non-linear & 76.7±1.5 & 0.28±0.01 & 0.29±0.02 & 0.12±0.01 & 1890.7 \\
2000 & Non-linear 1 & ER-20 & GES & 120.0±0.0 & 0.46±0.00 & 0.68±0.00 & 1.00±0.00 & 42.5 \\
\addlinespace
2000 & Non-linear 1 & ER-100 & DAG-GNN & 993.7±20.4 & 0.03±0.02 & 0.05±0.06 & \textbf{0.00±0.00} & 3035.9 \\
2000 & Non-linear 1 & ER-100 & DDCD Linear & 963.7±5.5 & 0.11±0.00 & 0.34±0.02 & 0.01±0.00 & 19.3 \\
2000 & Non-linear 1 & ER-100 & DDCD Non-linear & \textbf{177.6±28.7} & \textbf{0.83±0.03} & \textbf{0.00±0.00} & \textbf{0.00±0.00} & \textbf{5.0} \\
2000 & Non-linear 1 & ER-100 & DDCD Smooth & 1986.3±463.1 & 0.20±0.11 & 0.87±0.03 & 0.36±0.17 & \textbf{5.0} \\
2000 & Non-linear 1 & ER-100 & GAE & 676.8±473.4 & 0.68±0.30 & 0.23±0.28 & 0.09±0.12 & 8688.7 \\
2000 & Non-linear 1 & ER-100 & GOLEM & 1039.0±0.0 & 0.01±0.00 & 0.71±0.00 & 0.01±0.00 & 1090.3 \\
2000 & Non-linear 1 & ER-100 & Notears & 1000.0±0.0 & 0.16±0.00 & 0.48±0.00 & 0.04±0.00 & 711.7 \\
2000 & Non-linear 1 & ER-100 & Notears Non-linear & 1348.1±26.7 & 0.42±0.01 & 0.65±0.01 & 0.20±0.01 & 37706.2 \\
2000 & Non-linear 1 & ER-100 & PC & 1177.0±0.0 & 0.14±0.00 & 0.73±0.00 & 0.10±0.00 & 2864.1 \\
2000 & Non-linear 1 & ER-100 & DAGMA Linear & 978.0±0.0 & 0.10±0.00 & 0.37±0.00 & 0.01±0.00 & 13.0 \\
2000 & Non-linear 1 & ER-100 & DAGMA Non-linear & 962.3±3.2 & 0.07±0.00 & 0.27±0.02 & 0.01±0.00 & 1508.5 \\
2000 & Non-linear 1 & ER-100 & GES & - & - & - & - & >12hr \\
\addlinespace
\addlinespace
20 & linear & SF-20 & DAG-GNN & 63.2±8.3 & \textbf{0.84±0.03} & 0.56±0.04 & 0.41±0.06 & 15.4 \\
20 & linear & SF-20 & DDCD Linear & 68.3±1.6 & \textbf{0.84±0.01} & 0.58±0.01 & 0.44±0.01 & 14.0 \\
20 & linear & SF-20 & DDCD Non-linear & \textbf{31.8±5.4} & 0.70±0.15 & \textbf{0.31±0.05} & 0.12±0.04 & 4.1 \\
20 & linear & SF-20 & DDCD Smooth & 98.4±10.9 & 0.17±0.09 & 0.89±0.05 & 0.56±0.11 & 3.9 \\
20 & linear & SF-20 & GAE & 67.5±18.1 & 0.15±0.12 & 0.73±0.29 & 0.20±0.14 & 454.2 \\
20 & linear & SF-20 & GOLEM & 77.0±0.0 & 0.56±0.00 & 0.66±0.00 & 0.41±0.00 & 106.8 \\
20 & linear & SF-20 & Notears & 57.0±0.0 & 0.50±0.00 & 0.56±0.00 & 0.24±0.00 & 25.2 \\
20 & linear & SF-20 & Notears Non-linear & 108.6±3.7 & 0.49±0.03 & 0.78±0.01 & 0.65±0.02 & 168.4 \\
20 & linear & SF-20 & PC & 57.0±0.0 & 0.10±0.00 & 0.75±0.00 & \textbf{0.11±0.00} & \textbf{0.0} \\
20 & linear & SF-20 & DAGMA Linear & 80.0±0.0 & 0.58±0.00 & 0.66±0.00 & 0.43±0.00 & 1.0 \\
20 & linear & SF-20 & DAGMA Non-linear & 105.3±9.1 & 0.34±0.04 & 0.81±0.02 & 0.54±0.08 & 189.6 \\
20 & linear & SF-20 & GES & 66.2±5.4 & 0.32±0.08 & 0.70±0.05 & 0.28±0.00 & 5.3 \\
\addlinespace
20 & linear & SF-100 & DAG-GNN & 729.3±175.7 & \textbf{0.51±0.06} & 0.77±0.05 & 0.12±0.04 & 186.0 \\
20 & linear & SF-100 & DDCD Linear & 519.6±74.2 & 0.47±0.17 & 0.63±0.22 & 0.08±0.03 & 19.0 \\
20 & linear & SF-100 & DDCD Non-linear & \textbf{290.0±30.3} & 0.14±0.16 & \textbf{0.13±0.15} & \textbf{0.00±0.00} & 5.1 \\
20 & linear & SF-100 & DDCD Smooth & 876.0±208.1 & 0.12±0.06 & 0.94±0.02 & 0.14±0.05 & 5.0 \\
20 & linear & SF-100 & GAE & 294.1±27.8 & 0.24±0.10 & 0.35±0.14 & 0.01±0.01 & 1018.2 \\
20 & linear & SF-100 & GOLEM & 436.0±0.0 & 0.20±0.00 & 0.74±0.00 & 0.04±0.00 & 580.8 \\
20 & linear & SF-100 & Notears & 450.0±0.0 & 0.27±0.00 & 0.72±0.00 & 0.05±0.00 & 2718.5 \\
20 & linear & SF-100 & Notears Non-linear & 606.0±25.7 & 0.13±0.02 & 0.90±0.02 & 0.08±0.01 & 17922.8 \\
20 & linear & SF-100 & PC & 362.0±0.0 & 0.06±0.00 & 0.80±0.00 & 0.02±0.00 & \textbf{3.0} \\
20 & linear & SF-100 & DAGMA Linear & 579.1±8.6 & 0.25±0.01 & 0.81±0.01 & 0.08±0.00 & 72.7 \\
20 & linear & SF-100 & DAGMA Non-linear & 540.3±99.6 & 0.13±0.01 & 0.84±0.10 & 0.06±0.02 & 461.1 \\
20 & linear & SF-100 & GES & - & - & - & - & >12hr \\
\addlinespace
20 & linear & ER-20 & DAG-GNN & \textbf{64.1±5.2} & 0.85±0.05 & \textbf{0.37±0.02} & 0.56±0.04 & 31.1 \\
20 & linear & ER-20 & DDCD Linear & 65.0±1.5 & \textbf{0.90±0.01} & 0.39±0.01 & 0.63±0.02 & 14.0 \\
20 & linear & ER-20 & DDCD Non-linear & 92.8±4.0 & 0.15±0.06 & \textbf{0.33±0.07} & 0.09±0.05 & 4.0 \\
20 & linear & ER-20 & DDCD Smooth & 134.3±11.9 & 0.25±0.16 & 0.81±0.08 & 1.07±0.34 & 4.0 \\
20 & linear & ER-20 & GAE & 90.3±5.5 & 0.29±0.16 & 0.40±0.04 & 0.22±0.12 & 353.7 \\
20 & linear & ER-20 & GOLEM & 81.0±0.0 & 0.65±0.00 & 0.42±0.00 & 0.52±0.00 & 112.7 \\
20 & linear & ER-20 & Notears & 85.7±0.9 & 0.37±0.01 & 0.41±0.01 & 0.29±0.00 & 46.3 \\
20 & linear & ER-20 & Notears Non-linear & 106.5±6.3 & 0.49±0.10 & 0.58±0.05 & 0.75±0.12 & 697.1 \\
20 & linear & ER-20 & PC & 101.0±0.0 & 0.09±0.00 & 0.65±0.00 & 0.19±0.00 & \textbf{0.0} \\
20 & linear & ER-20 & DAGMA Linear & 71.8±0.6 & 0.69±0.00 & 0.40±0.00 & 0.51±0.01 & 5.8 \\
20 & linear & ER-20 & DAGMA Non-linear & 100.0±0.0 & 0.00±0.00 & \textbf{0.00±0.00} & \textbf{0.00±0.00} & 10.7 \\
20 & linear & ER-20 & GES & 110.0±9.7 & 0.20±0.08 & 0.70±0.11 & 0.51±0.00 & 14.6 \\
\addlinespace
20 & linear & ER-100 & DAG-GNN & 1245.9±247.7 & 0.09±0.07 & 0.79±0.03 & 0.09±0.08 & 183.6 \\
20 & linear & ER-100 & DDCD Linear & 1014.0±1.2 & 0.02±0.00 & 0.74±0.00 & \textbf{0.01±0.00} & 19.8 \\
20 & linear & ER-100 & DDCD Non-linear & 1056.7±47.7 & 0.03±0.02 & \textbf{0.61±0.32} & 0.03±0.02 & \textbf{5.0} \\
20 & linear & ER-100 & DDCD Smooth & 2266.4±325.5 & 0.21±0.06 & 0.91±0.01 & 0.49±0.13 & \textbf{5.0} \\
20 & linear & ER-100 & GAE & 1121.8±84.3 & 0.05±0.03 & 0.72±0.26 & 0.05±0.03 & 2100.3 \\
20 & linear & ER-100 & GOLEM & 1117.0±0.0 & 0.02±0.00 & 0.89±0.00 & 0.04±0.00 & 590.7 \\
20 & linear & ER-100 & Notears & \textbf{994.0±0.0} & 0.01±0.00 & 0.68±0.00 & \textbf{0.01±0.00} & 45.2 \\
20 & linear & ER-100 & Notears Non-linear & 1001.8±10.6 & 0.00±0.00 & 0.97±0.07 & \textbf{0.01±0.00} & 224.1 \\
20 & linear & ER-100 & PC & 1013.0±0.0 & 0.03±0.00 & 0.72±0.00 & 0.02±0.00 & 6.9 \\
20 & linear & ER-100 & DAGMA Linear & 1767.3±0.9 & \textbf{0.33±0.00} & 0.77±0.00 & 0.28±0.00 & 90.7 \\
20 & linear & ER-100 & DAGMA Non-linear & 980.0±0.0 & 0.00±0.00 & \textbf{0.00±0.00} & \textbf{0.00±0.00} & 40.7 \\
20 & linear & ER-100 & GES & - & - & - & - & >12hr \\
\addlinespace
20 & Non-linear 1 & SF-20 & DAG-GNN & 76.4±14.1 & 0.48±0.09 & 0.65±0.07 & 0.36±0.12 & 14.8 \\
20 & Non-linear 1 & SF-20 & DDCD Linear & 74.1±3.9 & \textbf{0.69±0.02} & 0.62±0.02 & 0.43±0.03 & 14.0 \\
20 & Non-linear 1 & SF-20 & DDCD Non-linear & 53.4±5.8 & 0.33±0.05 & \textbf{0.51±0.08} & 0.14±0.05 & 4.0 \\
20 & Non-linear 1 & SF-20 & DDCD Smooth & 84.3±10.7 & 0.26±0.14 & 0.82±0.10 & 0.44±0.09 & 4.0 \\
20 & Non-linear 1 & SF-20 & GAE & 88.9±35.0 & 0.26±0.21 & 0.82±0.15 & 0.46±0.46 & 406.0 \\
20 & Non-linear 1 & SF-20 & GOLEM & 85.0±0.0 & 0.60±0.00 & 0.68±0.00 & 0.47±0.00 & 109.2 \\
20 & Non-linear 1 & SF-20 & Notears & 53.0±0.0 & 0.46±0.00 & 0.56±0.00 & 0.22±0.00 & 9.2 \\
20 & Non-linear 1 & SF-20 & Notears Non-linear & 102.4±4.1 & 0.64±0.05 & 0.74±0.02 & 0.68±0.03 & 26.3 \\
20 & Non-linear 1 & SF-20 & PC & \textbf{45.0±0.0} & 0.19±0.00 & \textbf{0.29±0.00} & \textbf{0.03±0.00} & \textbf{0.0} \\
20 & Non-linear 1 & SF-20 & DAGMA Linear & 72.0±0.0 & 0.56±0.00 & 0.65±0.00 & 0.38±0.00 & 1.0 \\
20 & Non-linear 1 & SF-20 & DAGMA Non-linear & 111.4±4.4 & 0.53±0.06 & 0.78±0.02 & 0.72±0.02 & 239.1 \\
20 & Non-linear 1 & SF-20 & GES & 66.2±8.7 & 0.19±0.04 & 0.75±0.06 & 0.22±0.00 & 5.9 \\
\addlinespace
20 & Non-linear 1 & SF-100 & DAG-GNN & 798.6±149.3 & \textbf{0.36±0.08} & 0.84±0.03 & 0.13±0.04 & 185.6 \\
20 & Non-linear 1 & SF-100 & DDCD Linear & 686.4±10.8 & 0.06±0.00 & 0.96±0.00 & 0.08±0.00 & 19.0 \\
20 & Non-linear 1 & SF-100 & DDCD Non-linear & 502.6±96.6 & 0.07±0.05 & 0.90±0.08 & 0.04±0.02 & 5.0 \\
20 & Non-linear 1 & SF-100 & DDCD Smooth & 1254.0±135.8 & 0.19±0.14 & 0.94±0.04 & 0.22±0.02 & 5.0 \\
20 & Non-linear 1 & SF-100 & GAE & 421.1±154.6 & 0.20±0.08 & \textbf{0.66±0.17} & 0.04±0.04 & 1123.3 \\
20 & Non-linear 1 & SF-100 & GOLEM & 400.0±0.0 & 0.01±0.00 & 0.96±0.00 & \textbf{0.02±0.00} & 553.0 \\
20 & Non-linear 1 & SF-100 & Notears & 542.0±0.0 & 0.13±0.00 & 0.87±0.00 & 0.06±0.00 & 570.4 \\
20 & Non-linear 1 & SF-100 & Notears Non-linear & 764.6±14.1 & 0.13±0.01 & 0.93±0.00 & 0.11±0.00 & 23053.8 \\
20 & Non-linear 1 & SF-100 & PC & \textbf{382.0±0.0} & 0.04±0.00 & 0.88±0.00 & \textbf{0.02±0.00} & \textbf{1.0} \\
20 & Non-linear 1 & SF-100 & DAGMA Linear & 633.2±2.3 & 0.13±0.01 & 0.90±0.00 & 0.08±0.00 & 42.6 \\
20 & Non-linear 1 & SF-100 & DAGMA Non-linear & 385.0±7.7 & 0.07±0.01 & 0.81±0.02 & \textbf{0.02±0.00} & 503.2 \\
20 & Non-linear 1 & SF-100 & GES & - & - & - & - & >12hr \\
\addlinespace
20 & Non-linear 1 & ER-20 & DAG-GNN & 99.6±4.2 & 0.39±0.06 & 0.57±0.03 & 0.54±0.09 & 34.6 \\
20 & Non-linear 1 & ER-20 & DDCD Linear & 96.8±3.8 & \textbf{0.53±0.03} & 0.55±0.02 & 0.67±0.03 & 13.8 \\
20 & Non-linear 1 & ER-20 & DDCD Non-linear & \textbf{86.0±8.7} & 0.24±0.08 & \textbf{0.42±0.10} & 0.19±0.08 & 4.0 \\
20 & Non-linear 1 & ER-20 & DDCD Smooth & 111.1±6.1 & 0.27±0.08 & 0.72±0.06 & 0.71±0.16 & 4.0 \\
20 & Non-linear 1 & ER-20 & GAE & 105.0±6.6 & 0.23±0.12 & 0.67±0.06 & 0.45±0.18 & 713.1 \\
20 & Non-linear 1 & ER-20 & GOLEM & 112.0±0.0 & 0.36±0.00 & 0.63±0.00 & 0.64±0.00 & 272.4 \\
20 & Non-linear 1 & ER-20 & Notears & 93.0±0.0 & 0.35±0.00 & 0.56±0.00 & 0.46±0.00 & 3.6 \\
20 & Non-linear 1 & ER-20 & Notears Non-linear & 110.4±8.5 & 0.48±0.06 & 0.65±0.05 & 0.89±0.07 & 94.9 \\
20 & Non-linear 1 & ER-20 & PC & 98.0±0.0 & 0.05±0.00 & 0.71±0.00 & \textbf{0.13±0.00} & \textbf{0.0} \\
20 & Non-linear 1 & ER-20 & DAGMA Linear & 105.0±0.0 & 0.36±0.00 & 0.61±0.00 & 0.59±0.00 & 1.0 \\
20 & Non-linear 1 & ER-20 & DAGMA Non-linear & 111.9±10.8 & 0.43±0.09 & 0.67±0.06 & 0.88±0.07 & 183.2 \\
20 & Non-linear 1 & ER-20 & GES & 106.8±4.1 & 0.12±0.04 & 0.74±0.05 & 0.36±0.00 & 6.9 \\
\addlinespace
20 & Non-linear 1 & ER-100 & DAG-GNN & 1451.2±160.1 & \textbf{0.15±0.05} & 0.80±0.01 & 0.15±0.05 & 350.3 \\
20 & Non-linear 1 & ER-100 & DDCD Linear & 1273.9±8.3 & 0.11±0.00 & 0.77±0.01 & 0.09±0.00 & 19.2 \\
20 & Non-linear 1 & ER-100 & DDCD Non-linear & 1151.2±38.2 & 0.06±0.01 & \textbf{0.76±0.03} & 0.05±0.01 & 5.2 \\
20 & Non-linear 1 & ER-100 & DDCD Smooth & 2021.6±141.4 & 0.12±0.04 & 0.92±0.03 & 0.34±0.04 & 5.0 \\
20 & Non-linear 1 & ER-100 & GAE & 1231.2±67.6 & 0.07±0.03 & 0.81±0.05 & 0.08±0.02 & 1934.0 \\
20 & Non-linear 1 & ER-100 & GOLEM & 1273.0±0.0 & 0.07±0.00 & 0.82±0.00 & 0.08±0.00 & 750.7 \\
20 & Non-linear 1 & ER-100 & Notears & 1336.6±5.8 & 0.10±0.00 & 0.82±0.00 & 0.11±0.00 & 1425.9 \\
20 & Non-linear 1 & ER-100 & Notears Non-linear & 1352.8±44.8 & 0.09±0.01 & 0.83±0.02 & 0.12±0.01 & 10844.1 \\
20 & Non-linear 1 & ER-100 & PC & \textbf{1072.0±0.0} & 0.01±0.00 & 0.86±0.00 & \textbf{0.02±0.00} & \textbf{1.0} \\
20 & Non-linear 1 & ER-100 & DAGMA Linear & 1363.7±2.4 & 0.09±0.00 & 0.84±0.00 & 0.11±0.00 & 55.7 \\
20 & Non-linear 1 & ER-100 & DAGMA Non-linear & 1103.2±53.7 & 0.03±0.01 & 0.78±0.04 & 0.03±0.02 & 459.5 \\
20 & Non-linear 1 & ER-100 & GES & - & - & - & - & >12hr \\
\hline
\end{longtable}

\fontsize{11}{11}\selectfont

\subsection{Non-linear Diversity Experiments}

To address concerns about testing DDCD in a limited range of nonlinear functions, we conducted a comprehensive evaluation across seven distinct non-linear functions as shown in Figure \ref{fig:nonlinear_shd}. These functions consist of smooth ($\sin, \cos, \tanh$), non-smooth (ReLU), bounded (Sigmoid, $\tanh$), and unbounded (polynomial). Similar to the main experiment, we generated both SF and ER graphs at different sizes. Here we compared the performance among the Non-linear and Linear models from DDCD, DAGMA, and GES.

Based on our experiment, DDCD Non-Linear consistently achieves the lowest error rates (SHD) across most non-linear functions. Specifically, in the periodic (Types 1 \& 2) and activation-function-based relationships (Types 5, 6, \& 7), DDCD Non-Linear significantly outperforming linear baselines and often surpassing the state-of-the-art DAGMA Non-Linear (orange). However,  quadratic and sigmoid relationships (Types 3 \& 4) prove difficult for DDCD Non-linear but they seem to be challenging for all the other methods as well.

In terms of runtime, DDCD Non-Linear (blue) present sharp advantage compared with other methods, especially on large graphs. While DAGMA Non-Linear and GES exhibit exponential scaling and consume thousands of seconds on graph with 100 nodes, DDCD Non-Linear can finish within 20 seconds (on CPU). Note that we don't have data points from some GES runs because they run out of time (3 hrs). 

\begin{figure}[ht]
    \centering
    \includegraphics[width=0.8\textwidth]{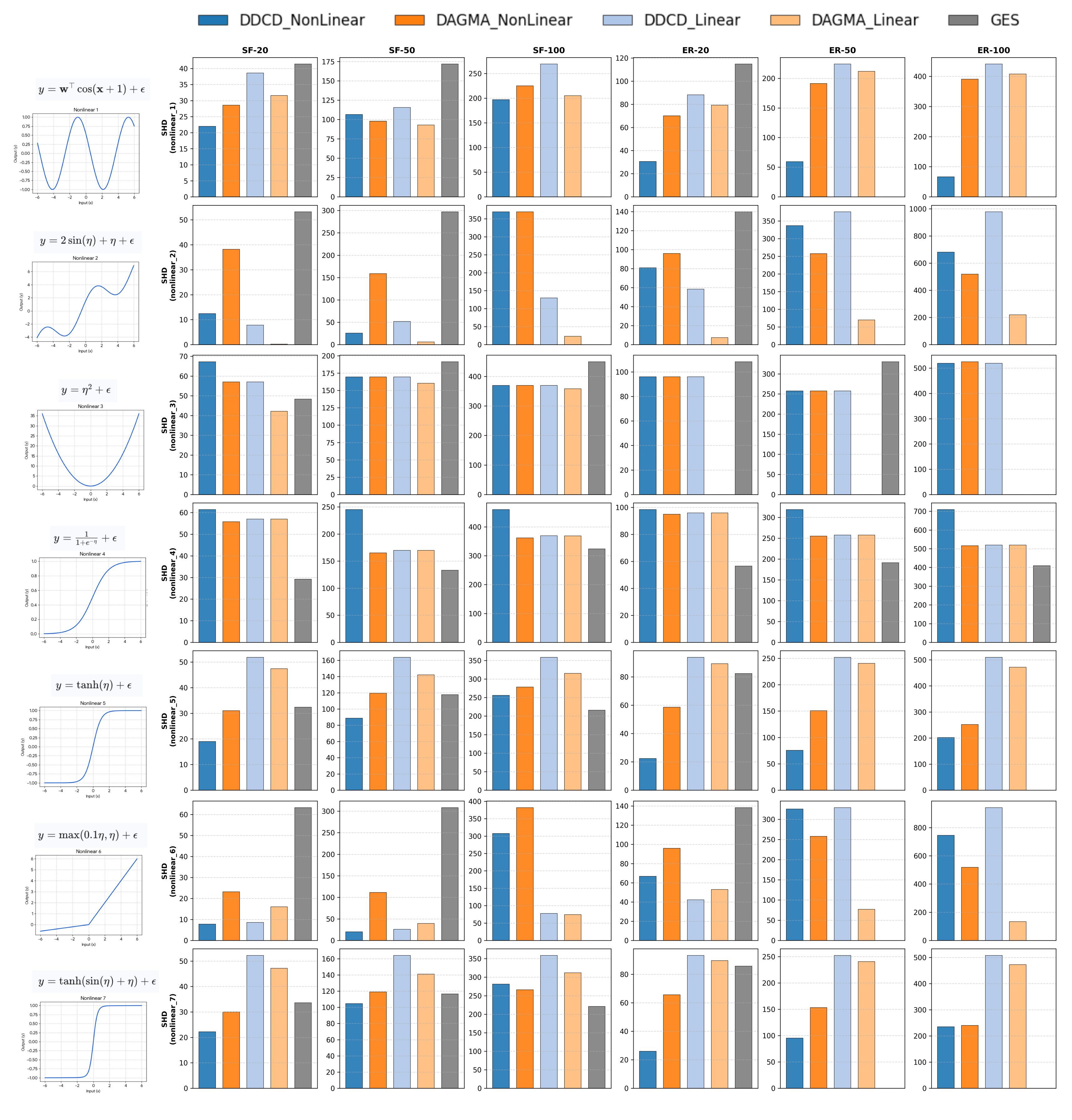}
    \caption{SHD Results from the Non-linear Diversity Experiments}
    \label{fig:nonlinear_shd}
\end{figure}

\begin{figure}[ht]
    \centering
    \includegraphics[width=0.8\textwidth]{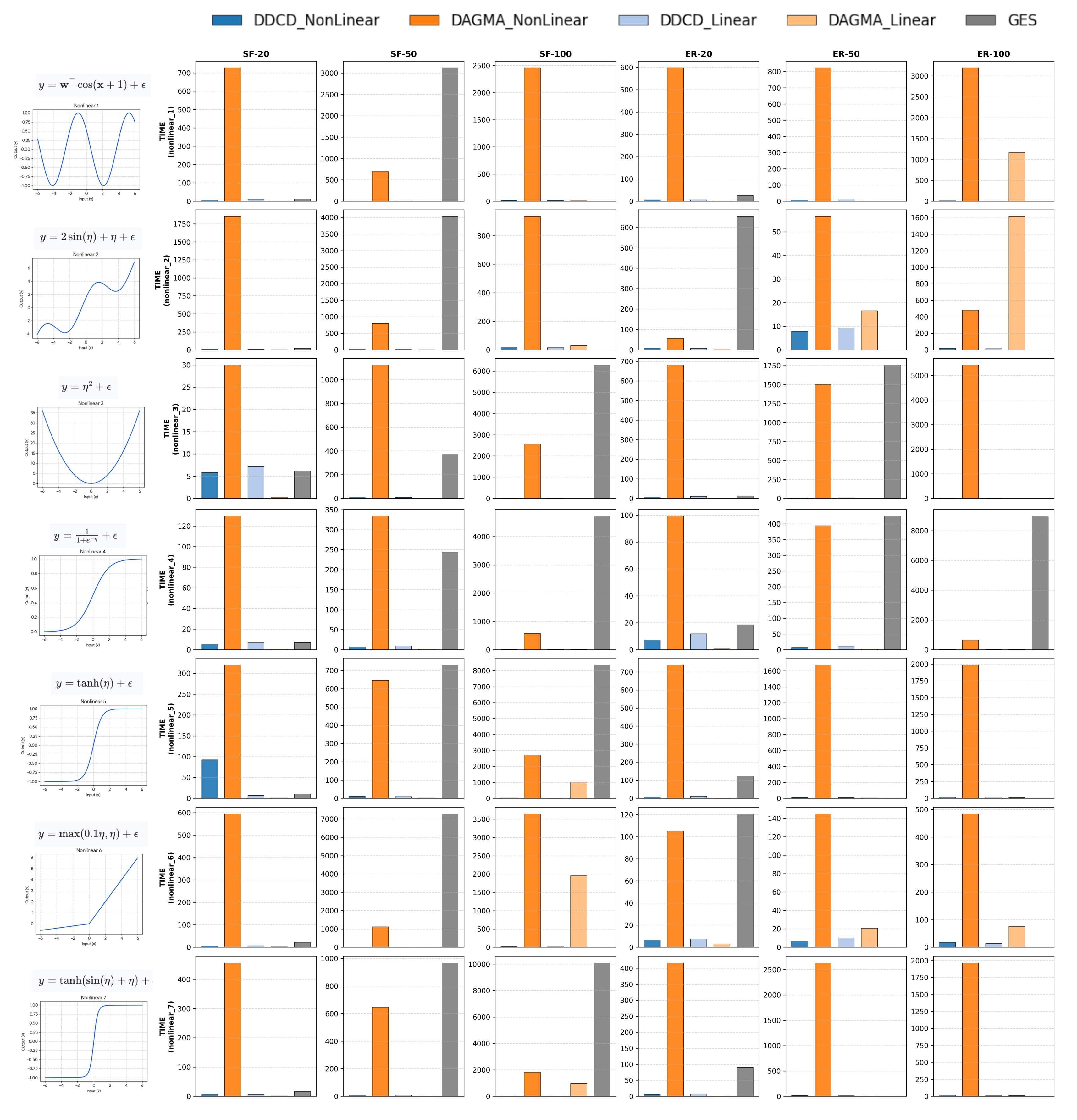}
    \caption{Time Costs from the Non-linear Diversity Experiments}
    \label{fig:nonlinear_time}
\end{figure}

\subsection{Case Study 2: Physical Function in Aging Population} \label{aging}

\begin{figure}[t]
\centering
\includegraphics[width=\linewidth]{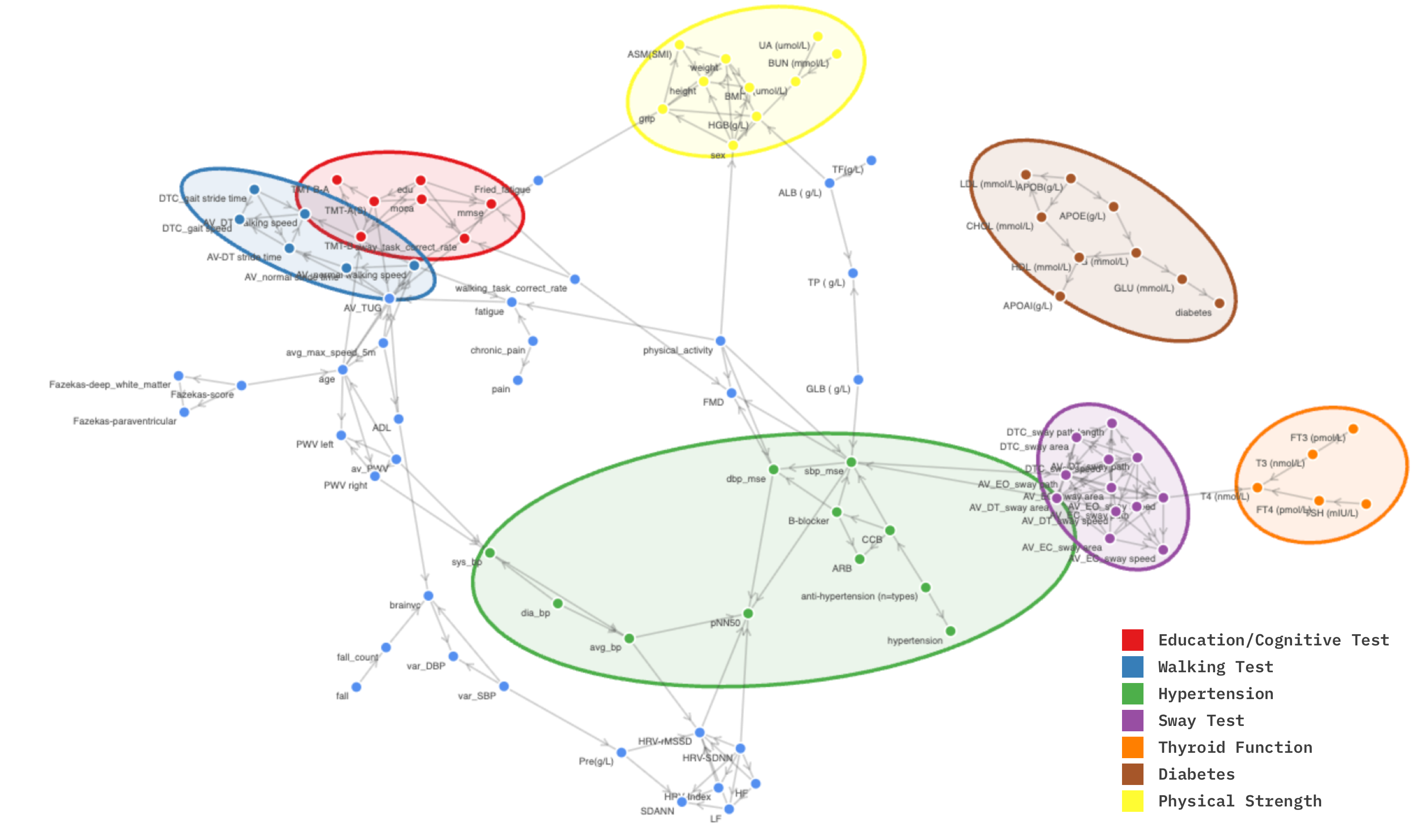}
\caption{Inferred Causal Network in an Aging Study. }
\label{fig:aging}
\end{figure}

In the aging dataset, the average age was 70.6 $\pm$ 7.5 years old. Within the cohort, 53.8\% of the participants were women; 60.6\% had hypertension, and 30.4\% had diabetes. The causal network was also inferred using DDCD smooth and a network was extracted with a threshold at 0.1. The network visualization was done using python package lightgraph \citep{lightgraph}.  

In Figure \ref{fig:aging}, we identified several correlated node clusters, which we circled and summarized in the legend.  These clusters suggest that inferred edges capture similar functions; many specific causal relationships in the clusters make sense.  For example, education (level) is inferred as a cause of performance scores on several cognitive tests, including the Trail Making Tests (TMT) \citep{bowie2006administration, tombaugh2004trail}, Mini-Mental State Examination (MMSE) \citep{tombaugh1992mini}, and the performance on a postural sway task \citep{zhou2017complexity}.  While the latter does not initially appear to be related to education level, postural sway is known to increase with increased cognitive load \citep{pellecchia2003postural}, suggesting an indirect causal relationship captured here. 
Current norms
for postural sway tests do not reflect education levels \citep{goble2018postural}, but our findings  suggest that
investigating this possibility might be valuable.

In the diabetes cluster, APOB and cholesterol (CHOL) appear as root causes. APOB secretion affects production of APOE \citep{shinozaki2025apob}, which binds to Triglycerides (TG), which contribute to Glucose Level (GLU), which %
is among the diagnostic criteria for diabetes. Between the clusters of hypertension and physical strength, we see nodes representing Albumin (ALB), Globulin (GLB) and Total Protein (TP). ALB and GLB are two main types of blood proteins, so the directions of these arrows are logical as well.  As in the MI data set above, the directions of a few inferred edges are reversed.  %
For example, physical activity was inferred as the cause of sex/gender.  %

\subsection{Inferred weights of the MI network}
\label{mi_edgeweights_ddcd}

Here we show a table containing the weights of an inferred network by DDCD Smooth for the MI experiment, sorted by edge weights, similar to that in Figure ~\ref{fig:mi}.  Most of the edges here seem explicable, reflecting standard medical understanding of correlations if not causality.  We do not claim that they represent novel findings, only that the existence of a directed ``causal'' edge in this graph is generally logical.  Therefore, investigating any edges whose explanations are {\em not} readily apparent in similar data sets may lead to new findings in the future.

\small

\begin{longtable}{l p{3.5cm} l p{3.5cm} l}
\caption{DDCD Smooth MI Edges}\\
\hline
\textbf{Source} & \textbf{Source Label} & \textbf{Target} & \textbf{Target Label} & \textbf{Weight} \\
\hline
\endfirsthead
\hline
\textbf{Source} & \textbf{Source Label} & \textbf{Target} & \textbf{Target Label} & \textbf{Weight} \\
\hline
\endhead
\hline
\endfoot
\endlastfoot
LID\_KB & Use of lidocaine by the Emergency Cardiology Team & NA\_KB & Use of opioid drugs by the Emergency Cardiology Team & 0.66\\
R\_AB\_2\_n & Relapse of the pain in the second day of the hospital period & NA\_R\_2\_n & Use of opioid drugs in the ICU in the second day of the hospital period & 0.58\\
R\_AB\_1\_n & Relapse of the pain in the first hours of the hospital period & NA\_R\_1\_n & Use of opioid drugs in the ICU in the first hours of the hospital period & 0.52\\
R\_AB\_3\_n & Relapse of the pain in the third day of the hospital period & NA\_R\_3\_n & Use of opioid drugs in the ICU in the third day of the hospital period & 0.50\\
NOT\_NA\_2\_n & Use of NSAIDs in the ICU in the second day of the hospital period & R\_AB\_2\_n & Relapse of the pain in the second day of the hospital period & 0.49\\
RAZRIV & Myocardial rupture & LET\_IS & Lethal outcome & 0.48\\
NA\_KB & Use of opioid drugs by the Emergency Cardiology Team & NOT\_NA\_KB & Use of NSAIDs by the Emergency Cardiology Team & 0.47\\
NOT\_NA\_2\_n & Use of NSAIDs in the ICU in the second day of the hospital period & NOT\_NA\_3\_n & Use of NSAIDs in the ICU in the third day of the hospital period & 0.47\\
NOT\_NA\_3\_n & Use of NSAIDs in the ICU in the third day of the hospital period & R\_AB\_3\_n & Relapse of the pain in the third day of the hospital period & 0.46\\
LID\_KB & Use of lidocaine by the Emergency Cardiology Team & NOT\_NA\_KB & Use of NSAIDs by the Emergency Cardiology Team & 0.44\\
K\_SH\_POST & Cardiogenic shock at the time of admission to intensive care unit & LET\_IS & Lethal outcome & 0.39\\
ASP\_S\_n & Use of acetylsalicylic acid in the ICU & GEPAR\_S\_n & Use of a anticoagulants (heparin) in the ICU & 0.38\\
FIBR\_JELUD & Ventricular fibrillation & LID\_S\_n & Use of lidocaine in the ICU & 0.32\\
n\_r\_ecg\_p\_03 & Premature ventricular contractions on ECG at the time of admission to hospital & LID\_S\_n & Use of lidocaine in the ICU & 0.32\\
ZSN & Chronic heart failure & ZSN\_A & Presence of chronic Heart failure (HF) in the anamnesis & 0.31\\
n\_r\_ecg\_p\_05 & Paroxysms of atrial fibrillation on ECG at the time of admission to hospital & FIBR\_PREDS & Atrial fibrillation & 0.29\\
n\_r\_ecg\_p\_04 & Frequent premature ventricular contractions on ECG at the time of admission to hospital & LID\_S\_n & Use of lidocaine in the ICU & 0.27\\
R\_AB\_1\_n & Relapse of the pain in the first hours of the hospital period & NOT\_NA\_1\_n & Use of NSAIDs in the ICU in the first hours of the hospital period & 0.26\\
R\_AB\_3\_n & Relapse of the pain in the third day of the hospital period & R\_AB\_2\_n & Relapse of the pain in the second day of the hospital period & 0.25\\
NOT\_NA\_2\_n & Use of NSAIDs in the ICU in the second day of the hospital period & R\_AB\_1\_n & Relapse of the pain in the first hours of the hospital period & 0.23\\
NITR\_S & Use of liquid nitrates in the ICU & OTEK\_LANC & Pulmonary edema & 0.22\\
NOT\_NA\_2\_n & Use of NSAIDs in the ICU in the second day of the hospital period & NOT\_NA\_1\_n & Use of NSAIDs in the ICU in the first hours of the hospital period & 0.21\\
LID\_S\_n & Use of lidocaine in the ICU & NA\_KB & Use of opioid drugs by the Emergency Cardiology Team & 0.19\\
R\_AB\_3\_n & Relapse of the pain in the third day of the hospital period & R\_AB\_1\_n & Relapse of the pain in the first hours of the hospital period & 0.19\\
NOT\_NA\_3\_n & Use of NSAIDs in the ICU in the third day of the hospital period & NOT\_NA\_1\_n & Use of NSAIDs in the ICU in the first hours of the hospital period & 0.19\\
zab\_leg\_01 & Chronic bronchitis in the anamnesis & ZSN & Chronic heart failure & 0.19\\
NA\_R\_2\_n & Use of opioid drugs in the ICU in the second day of the hospital period & R\_AB\_1\_n & Relapse of the pain in the first hours of the hospital period & 0.19\\
REC\_IM & Relapse of the myocardial infarction & R\_AB\_3\_n & Relapse of the pain in the third day of the hospital period & 0.18\\
n\_p\_ecg\_p\_12 & Complete RBBB on ECG at the time of admission to hospital & LET\_IS & Lethal outcome & 0.18\\
FIBR\_JELUD & Ventricular fibrillation & GIPO\_K & Hypokalemia ( < 4 mmol/L) & 0.18\\
NOT\_NA\_3\_n & Use of NSAIDs in the ICU in the third day of the hospital period & NA\_R\_3\_n & Use of opioid drugs in the ICU in the third day of the hospital period & 0.18\\
LET\_IS & Lethal outcome & ASP\_S\_n & Use of acetylsalicylic acid in the ICU & -0.20\\
TRENT\_S\_n & Use of Trental in the ICU & ASP\_S\_n & Use of acetylsalicylic acid in the ICU & -0.29\\

\end{longtable}

\begin{figure}[htb!]
\includegraphics[width=\textwidth]{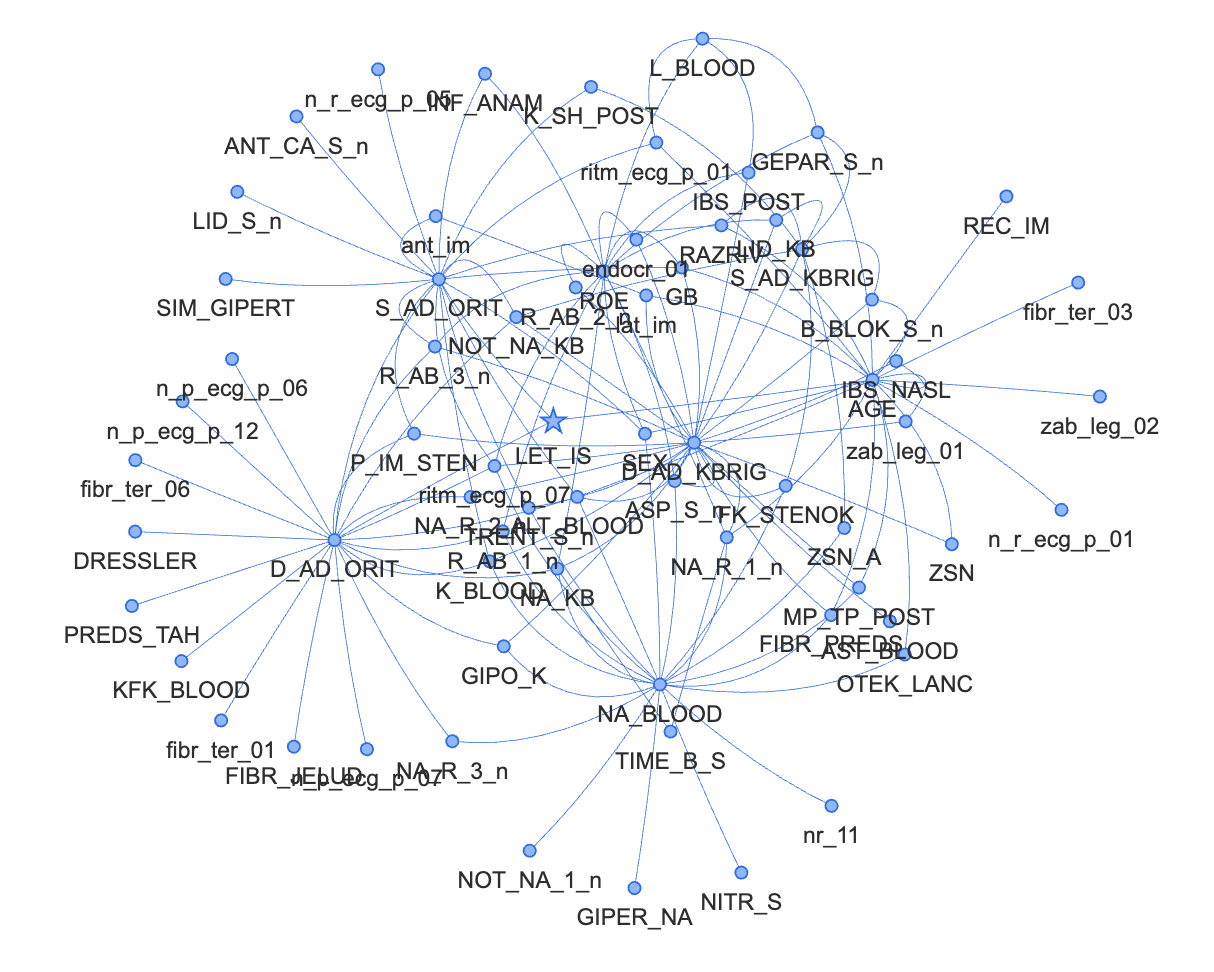}
\caption{Inferred Causal Network around Lethal Outcome in Myocardial Infarction by NOTEARS. }
\label{fig:notears_mi}
\end{figure}

\newpage

\subsection{Inferred MI network with NOTEARS} 
\label{mi_edgeweights_notears}

Next, we show an example of the inferred network by NOTEARS, using recommended parameters and focused on the same 2-hop neighborhood around the "lethal outcome" node, on the exact same MI dataset.  While some edges make sense, most of the edges seem to be clustered around a few hub targets and overall seem less ``causal."  Patient age is one of the main target nodes, supposedly caused by other clinical variables.  
Compare the DDCD results in Figure ~\ref{fig:mi} to the NOTEARS results on the same data shown in Figure ~\ref{fig:notears_mi} and whose edges are detailed in the table below.

This excessively hub-and-spoke structure, with fewer separate clusters of connected nodes, in our experience reflects higher numbers of spurious connections.   The intuition behind this claim is that multiple causes of one outcome (and thus edges connected to hubs) will also, if real, be detected as causally related to one another.  While we have removed short cycles from these models, there are enough related causal terms to form directed causal clusters, as highlighted in the ovals in Figure ~\ref{fig:mi}, rather than forming hubs with many unrelated spoke nodes, as we see with NOTEARS.

\begin{longtable}{l p{3.5cm} l p{3.5cm} l}
\caption{NOTEARS MI Edges}\\
\hline
\textbf{Source} & \textbf{Source Label} & \textbf{Target} & \textbf{Target Label} & \textbf{Weight} \\
\hline
\endfirsthead
\hline
\textbf{Source} & \textbf{Source Label} & \textbf{Target} & \textbf{Target Label} & \textbf{Weight} \\
\hline
\endhead
\hline
\endfoot
\endlastfoot
LET\_IS & Lethal outcome & AGE & Age of patient. & 1 \\
LET\_IS & Lethal outcome & S\_AD\_ORIT & Systolic blood pressure according to intensive care unit & 1 \\
LET\_IS & Lethal outcome & D\_AD\_ORIT & Diastolic blood pressure according to intensive care unit & 1 \\
FK\_STENOK & Functional class (FC) of angina pectoris in the last year. & NA\_BLOOD & Serum sodium content & 1 \\
MP\_TP\_POST & Paroxysms of atrial fibrillation at the time of admission to intensive care unit, (or at a pre-hospital stage) & NA\_BLOOD & Serum sodium content & 1 \\
nr\_11 & Observing of arrhythmia in the anamnesis & NA\_BLOOD & Serum sodium content & 1 \\
ZSN\_A & Presence of chronic Heart failure (HF) in the anamnesis & NA\_BLOOD & Serum sodium content & 1 \\
OTEK\_LANC & Pulmonary edema & NA\_BLOOD & Serum sodium content & 1 \\
K\_BLOOD & Serum potassium content & NA\_BLOOD & Serum sodium content & 1 \\
FIBR\_PREDS & Atrial fibrillation & NA\_BLOOD & Serum sodium content & 1 \\
TRENT\_S\_n & Use of Trental in the ICU & NA\_BLOOD & Serum sodium content & 1 \\
ASP\_S\_n & Use of acetylsalicylic acid in the ICU & NA\_BLOOD & Serum sodium content & 1 \\
NOT\_NA\_1\_n & Use of NSAIDs in the ICU in the first hours of the hospital period & NA\_BLOOD & Serum sodium content & 1 \\
GIPO\_K & Hypokalemia ( < 4 mmol/L) & NA\_BLOOD & Serum sodium content & 1 \\
NA\_R\_3\_n & Use of opioid drugs in the ICU in the third day of the hospital period & NA\_BLOOD & Serum sodium content & 1 \\
NA\_R\_1\_n & Use of opioid drugs in the ICU in the first hours of the hospital period & NA\_BLOOD & Serum sodium content & 1 \\
NITR\_S & Use of liquid nitrates in the ICU & NA\_BLOOD & Serum sodium content & 1 \\
NA\_KB & Use of opioid drugs by the Emergency Cardiology Team & NA\_BLOOD & Serum sodium content & 1 \\
ritm\_ecg\_p\_07 & ECG rhythm at the time of admission to hospital: sinus with a heart rate above 90 (tachycardia) & NA\_BLOOD & Serum sodium content & 1 \\
R\_AB\_1\_n & Relapse of the pain in the first hours of the hospital period & NA\_BLOOD & Serum sodium content & 1 \\
ALT\_BLOOD & Serum AlAT content (ALT\_BLOOD) & NA\_BLOOD & Serum sodium content & 1 \\
GIPER\_NA & Increase of sodium in serum (more than 150 mmol/L) & NA\_BLOOD & Serum sodium content & 1 \\
SEX & Male & NA\_BLOOD & Serum sodium content & 1 \\
R\_AB\_3\_n & Relapse of the pain in the third day of the hospital period & ROE & ESR (Erythrocyte sedimentation rate) & 1 \\
R\_AB\_3\_n & Relapse of the pain in the third day of the hospital period & S\_AD\_ORIT & Systolic blood pressure according to intensive care unit & 1 \\
R\_AB\_3\_n & Relapse of the pain in the third day of the hospital period & D\_AD\_KBRIG & Diastolic blood pressure according to Emergency Cardiology Team & 1 \\
R\_AB\_3\_n & Relapse of the pain in the third day of the hospital period & D\_AD\_ORIT & Diastolic blood pressure according to intensive care unit & 1 \\
IBS\_NASL & Heredity on CHD & AGE & Age of patient. & 1 \\
zab\_leg\_02 & Obstructive chronic bronchitis in the anamnesis & AGE & Age of patient. & 1 \\
zab\_leg\_01 & Chronic bronchitis in the anamnesis & AGE & Age of patient. & 1 \\
GB & Presence of an essential hypertension & AGE & Age of patient. & 1 \\
fibr\_ter\_03 & Fibrinolytic therapy by Celiasum 3m IU & AGE & Age of patient. & 1 \\
SEX & Male & AGE & Age of patient. & 1 \\
MP\_TP\_POST & Paroxysms of atrial fibrillation at the time of admission to intensive care unit, (or at a pre-hospital stage) & AGE & Age of patient. & 1 \\
lat\_im & Presence of a lateral myocardial infarction (left ventricular)  & AGE & Age of patient. & 1 \\
ritm\_ecg\_p\_01 & ECG rhythm at the time of admission to hospital: sinus (with a heart rate 60-90) & AGE & Age of patient. & 1 \\
n\_r\_ecg\_p\_01 & Premature atrial contractions on ECG at the time of admission to hospital & AGE & Age of patient. & 1 \\
ALT\_BLOOD & Serum AlAT content (ALT\_BLOOD) & AGE & Age of patient. & 1 \\
NA\_R\_1\_n & Use of opioid drugs in the ICU in the first hours of the hospital period & AGE & Age of patient. & 1 \\
B\_BLOK\_S\_n & Use of beta-blockers in the ICU & AGE & Age of patient. & 1 \\
GEPAR\_S\_n & Use of a anticoagulants (heparin) in the ICU & AGE & Age of patient. & 1 \\
FIBR\_PREDS & Atrial fibrillation & AGE & Age of patient. & 1 \\
OTEK\_LANC & Pulmonary edema & AGE & Age of patient. & 1 \\
RAZRIV & Myocardial rupture & AGE & Age of patient. & 1 \\
ZSN & Chronic heart failure & AGE & Age of patient. & 1 \\
REC\_IM & Relapse of the myocardial infarction & AGE & Age of patient. & 1 \\
LET\_IS & Lethal outcome & AGE & Age of patient. & 1 \\
NITR\_S & Use of liquid nitrates in the ICU & NA\_BLOOD & Serum sodium content & 1 \\
GEPAR\_S\_n & Use of a anticoagulants (heparin) in the ICU & L\_BLOOD & White blood cell count & 1 \\
IBS\_POST & Coronary heart disease (CHD) in recent weeks, days before admission to hospital & L\_BLOOD & White blood cell count & 1 \\
ritm\_ecg\_p\_01 & ECG rhythm at the time of admission to hospital: sinus (with a heart rate 60-90) & L\_BLOOD & White blood cell count & 1 \\
L\_BLOOD & White blood cell count & ROE & ESR (Erythrocyte sedimentation rate) & 1 \\
NA\_KB & Use of opioid drugs by the Emergency Cardiology Team & D\_AD\_KBRIG & Diastolic blood pressure according to Emergency Cardiology Team & 1 \\
NA\_KB & Use of opioid drugs by the Emergency Cardiology Team & ROE & ESR (Erythrocyte sedimentation rate) & 1 \\
NA\_KB & Use of opioid drugs by the Emergency Cardiology Team & TIME\_B\_S & Time elapsed from the beginning of the attack of CHD to the hospital & 1 \\
NA\_KB & Use of opioid drugs by the Emergency Cardiology Team & NA\_BLOOD & Serum sodium content & 1 \\
NA\_KB & Use of opioid drugs by the Emergency Cardiology Team & D\_AD\_ORIT & Diastolic blood pressure according to intensive care unit & 1 \\
ALT\_BLOOD & Serum AlAT content (ALT\_BLOOD) & AGE & Age of patient. & 1 \\
ALT\_BLOOD & Serum AlAT content (ALT\_BLOOD) & D\_AD\_ORIT & Diastolic blood pressure according to intensive care unit & 1 \\
ALT\_BLOOD & Serum AlAT content (ALT\_BLOOD) & NA\_BLOOD & Serum sodium content & 1 \\
ALT\_BLOOD & Serum AlAT content (ALT\_BLOOD) & D\_AD\_KBRIG & Diastolic blood pressure according to Emergency Cardiology Team & 1 \\
ALT\_BLOOD & Serum AlAT content (ALT\_BLOOD) & S\_AD\_ORIT & Systolic blood pressure according to intensive care unit & 1 \\
SIM\_GIPERT & Symptomatic hypertension & S\_AD\_ORIT & Systolic blood pressure according to intensive care unit & 1 \\
INF\_ANAM & Quantity of myocardial infarctions in the anamnesis. & S\_AD\_ORIT & Systolic blood pressure according to intensive care unit & 1 \\
GB & Presence of an essential hypertension & S\_AD\_ORIT & Systolic blood pressure according to intensive care unit & 1 \\
SEX & Male & S\_AD\_ORIT & Systolic blood pressure according to intensive care unit & 1 \\
D\_AD\_ORIT & Diastolic blood pressure according to intensive care unit & S\_AD\_ORIT & Systolic blood pressure according to intensive care unit & 1 \\
K\_SH\_POST & Cardiogenic shock at the time of admission to intensive care unit & S\_AD\_ORIT & Systolic blood pressure according to intensive care unit & 1 \\
ant\_im & Presence of an anterior myocardial infarction (left ventricular)  & S\_AD\_ORIT & Systolic blood pressure according to intensive care unit & 1 \\
ritm\_ecg\_p\_01 & ECG rhythm at the time of admission to hospital: sinus (with a heart rate 60-90) & S\_AD\_ORIT & Systolic blood pressure according to intensive care unit & 1 \\
ritm\_ecg\_p\_07 & ECG rhythm at the time of admission to hospital: sinus with a heart rate above 90 (tachycardia) & S\_AD\_ORIT & Systolic blood pressure according to intensive care unit & 1 \\
n\_r\_ecg\_p\_05 & Paroxysms of atrial fibrillation on ECG at the time of admission to hospital & S\_AD\_ORIT & Systolic blood pressure according to intensive care unit & 1 \\
P\_IM\_STEN & Post-infarction angina & S\_AD\_ORIT & Systolic blood pressure according to intensive care unit & 1 \\
K\_BLOOD & Serum potassium content & S\_AD\_ORIT & Systolic blood pressure according to intensive care unit & 1 \\
ALT\_BLOOD & Serum AlAT content (ALT\_BLOOD) & S\_AD\_ORIT & Systolic blood pressure according to intensive care unit & 1 \\
R\_AB\_3\_n & Relapse of the pain in the third day of the hospital period & S\_AD\_ORIT & Systolic blood pressure according to intensive care unit & 1 \\
NOT\_NA\_KB & Use of NSAIDs by the Emergency Cardiology Team & S\_AD\_ORIT & Systolic blood pressure according to intensive care unit & 1 \\
LID\_S\_n & Use of lidocaine in the ICU & S\_AD\_ORIT & Systolic blood pressure according to intensive care unit & 1 \\
ANT\_CA\_S\_n & Use of calcium channel blockers in the ICU & S\_AD\_ORIT & Systolic blood pressure according to intensive care unit & 1 \\
TRENT\_S\_n & Use of Trental in the ICU & S\_AD\_ORIT & Systolic blood pressure according to intensive care unit & 1 \\
RAZRIV & Myocardial rupture & S\_AD\_ORIT & Systolic blood pressure according to intensive care unit & 1 \\
LET\_IS & Lethal outcome & S\_AD\_ORIT & Systolic blood pressure according to intensive care unit & 1 \\
IBS\_POST & Coronary heart disease (CHD) in recent weeks, days before admission to hospital & ROE & ESR (Erythrocyte sedimentation rate) & 1 \\
GB & Presence of an essential hypertension & ROE & ESR (Erythrocyte sedimentation rate) & 1 \\
INF\_ANAM & Quantity of myocardial infarctions in the anamnesis. & ROE & ESR (Erythrocyte sedimentation rate) & 1 \\
L\_BLOOD & White blood cell count & ROE & ESR (Erythrocyte sedimentation rate) & 1 \\
NA\_R\_1\_n & Use of opioid drugs in the ICU in the first hours of the hospital period & ROE & ESR (Erythrocyte sedimentation rate) & 1 \\
endocr\_01 & Diabetes mellitus in the anamnesis & ROE & ESR (Erythrocyte sedimentation rate) & 1 \\
ritm\_ecg\_p\_07 & ECG rhythm at the time of admission to hospital: sinus with a heart rate above 90 (tachycardia) & ROE & ESR (Erythrocyte sedimentation rate) & 1 \\
ASP\_S\_n & Use of acetylsalicylic acid in the ICU & ROE & ESR (Erythrocyte sedimentation rate) & 1 \\
GEPAR\_S\_n & Use of an anticoagulants (heparin) in the ICU & ROE & ESR (Erythrocyte sedimentation rate) & 1 \\
R\_AB\_1\_n & Relapse of the pain in the first hours of the hospital period & ROE & ESR (Erythrocyte sedimentation rate) & 1 \\
R\_AB\_2\_n & Relapse of the pain in the second day of the hospital period & ROE & ESR (Erythrocyte sedimentation rate) & 1 \\
R\_AB\_3\_n & Relapse of the pain in the third day of the hospital period & ROE & ESR (Erythrocyte sedimentation rate) & 1 \\
NA\_KB & Use of opioid drugs by the Emergency Cardiology Team & ROE & ESR (Erythrocyte sedimentation rate) & 1 \\
LID\_KB & Use of lidocaine by the Emergency Cardiology Team & ROE & ESR (Erythrocyte sedimentation rate) & 1 \\
SEX & Male & ROE & ESR (Erythrocyte sedimentation rate) & 1 \\
NA\_R\_2\_n & Use of opioid drugs in the ICU in the second day of the hospital period & D\_AD\_ORIT & Diastolic blood pressure according to intensive care unit & 1 \\
PREDS\_TAH & Supraventricular tachycardia & D\_AD\_ORIT & Diastolic blood pressure according to intensive care unit & 1 \\
FIBR\_JELUD & Ventricular fibrillation & D\_AD\_ORIT & Diastolic blood pressure according to intensive care unit & 1 \\
NA\_R\_3\_n & Use of opioid drugs in the ICU in the third day of the hospital period & D\_AD\_ORIT & Diastolic blood pressure according to intensive care unit & 1 \\
ritm\_ecg\_p\_07 & ECG rhythm at the time of admission to hospital: sinus with a heart rate above 90 (tachycardia) & D\_AD\_ORIT & Diastolic blood pressure according to intensive care unit & 1 \\
n\_p\_ecg\_p\_06 & Third-degree AV block on ECG at the time of admission to hospital & D\_AD\_ORIT & Diastolic blood pressure according to intensive care unit & 1 \\
NOT\_NA\_KB & Use of NSAIDs by the Emergency Cardiology Team & D\_AD\_ORIT & Diastolic blood pressure according to intensive care unit & 1 \\
NA\_KB & Use of opioid drugs by the Emergency Cardiology Team & D\_AD\_ORIT & Diastolic blood pressure according to intensive care unit & 1 \\
R\_AB\_3\_n & Relapse of the pain in the third day of the hospital period & D\_AD\_ORIT & Diastolic blood pressure according to intensive care unit & 1 \\
R\_AB\_1\_n & Relapse of the pain in the first hours of the hospital period & D\_AD\_ORIT & Diastolic blood pressure according to intensive care unit & 1 \\
KFK\_BLOOD & Serum CPK content & D\_AD\_ORIT & Diastolic blood pressure according to intensive care unit & 1 \\
ALT\_BLOOD & Serum AlAT content (ALT\_BLOOD) & D\_AD\_ORIT & Diastolic blood pressure according to intensive care unit & 1 \\
K\_BLOOD & Serum potassium content & D\_AD\_ORIT & Diastolic blood pressure according to intensive care unit & 1 \\
GIPO\_K & Hypokalemia ( < 4 mmol/L) & D\_AD\_ORIT & Diastolic blood pressure according to intensive care unit & 1 \\
fibr\_ter\_06 & Fibrinolytic therapy by Celiasum 500k & D\_AD\_ORIT & Diastolic blood pressure according to intensive care unit & 1 \\
P\_IM\_STEN & Post-infarction angina & D\_AD\_ORIT & Diastolic blood pressure according to intensive care unit & 1 \\
fibr\_ter\_01 & Fibrinolytic therapy by Celiasum 750k & D\_AD\_ORIT & Diastolic blood pressure according to intensive care unit & 1 \\
n\_p\_ecg\_p\_12 & Complete RBBB on ECG at the time of admission to hospital & D\_AD\_ORIT & Diastolic blood pressure according to intensive care unit & 1 \\
n\_p\_ecg\_p\_07 & LBBB (anterior branch) on ECG at the time of admission to hospital & D\_AD\_ORIT & Diastolic blood pressure according to intensive care unit & 1 \\
LET\_IS & Lethal outcome & D\_AD\_ORIT & Diastolic blood pressure according to intensive care unit & 1 \\
R\_AB\_1\_n & Relapse of the pain in the first hours of the hospital period & ROE & ESR (Erythrocyte sedimentation rate) & 1 \\
R\_AB\_1\_n & Relapse of the pain in the first hours of the hospital period & NA\_BLOOD & Serum sodium content & 1 \\
R\_AB\_1\_n & Relapse of the pain in the first hours of the hospital period & D\_AD\_ORIT & Diastolic blood pressure according to intensive care unit & 1 \\
AST\_BLOOD & Serum AsAT content & D\_AD\_KBRIG & Diastolic blood pressure according to Emergency Cardiology Team & 1 \\
R\_AB\_2\_n & Relapse of the pain in the second day of the hospital period & ROE & ESR (Erythrocyte sedimentation rate) & 1 \\
R\_AB\_2\_n & Relapse of the pain in the second day of the hospital period & D\_AD\_KBRIG & Diastolic blood pressure according to Emergency Cardiology Team & 1 \\
NA\_KB & Use of opioid drugs by the Emergency Cardiology Team & TIME\_B\_S & Time elapsed from the beginning of the attack of CHD to the hospital & 1 \\
NA\_R\_1\_n & Use of opioid drugs in the ICU in the first hours of the hospital period & TIME\_B\_S & Time elapsed from the beginning of the attack of CHD to the hospital & 1 \\
K\_BLOOD & Serum potassium content & S\_AD\_ORIT & Systolic blood pressure according to intensive care unit & 1 \\
K\_BLOOD & Serum potassium content & NA\_BLOOD & Serum sodium content & 1 \\
K\_BLOOD & Serum potassium content & D\_AD\_KBRIG & Diastolic blood pressure according to Emergency Cardiology Team & 1 \\
K\_BLOOD & Serum potassium content & D\_AD\_ORIT & Diastolic blood pressure according to intensive care unit & 1 \\
NOT\_NA\_KB & Use of NSAIDs by the Emergency Cardiology Team & S\_AD\_ORIT & Systolic blood pressure according to intensive care unit & 1 \\
NOT\_NA\_KB & Use of NSAIDs by the Emergency Cardiology Team & S\_AD\_KBRIG & Systolic blood pressure according to Emergency Cardiology Team & 1 \\
NOT\_NA\_KB & Use of NSAIDs by the Emergency Cardiology Team & D\_AD\_KBRIG & Diastolic blood pressure according to Emergency Cardiology Team & 1 \\
NOT\_NA\_KB & Use of NSAIDs by the Emergency Cardiology Team & D\_AD\_ORIT & Diastolic blood pressure according to intensive care unit & 1 \\
LID\_KB & Use of lidocaine by the Emergency Cardiology Team & S\_AD\_KBRIG & Systolic blood pressure according to Emergency Cardiology Team & 1 \\
LID\_KB & Use of lidocaine by the Emergency Cardiology Team & D\_AD\_KBRIG & Diastolic blood pressure according to Emergency Cardiology Team & 1 \\
LID\_KB & Use of lidocaine by the Emergency Cardiology Team & ROE & ESR (Erythrocyte sedimentation rate) & 1 \\
DRESSLER & Dressler syndrome & D\_AD\_ORIT & Diastolic blood pressure according to intensive care unit & 1 \\
KFK\_BLOOD & Serum CPK content & D\_AD\_ORIT & Diastolic blood pressure according to intensive care unit & 1 \\
NA\_R\_1\_n & Use of opioid drugs in the ICU in the first hours of the hospital period & D\_AD\_KBRIG & Diastolic blood pressure according to Emergency Cardiology Team & 1 \\
NA\_R\_2\_n & Use of opioid drugs in the ICU in the second day of the hospital period & D\_AD\_KBRIG & Diastolic blood pressure according to Emergency Cardiology Team & 1 \\
zab\_leg\_01 & Chronic bronchitis in the anamnesis & D\_AD\_KBRIG & Diastolic blood pressure according to Emergency Cardiology Team & 1 \\
ZSN\_A & Presence of chronic Heart failure (HF) in the anamnesis & S\_AD\_KBRIG & Systolic blood pressure according to Emergency Cardiology Team & 1 \\
endocr\_01 & Diabetes mellitus in the anamnesis & D\_AD\_KBRIG & Diastolic blood pressure according to Emergency Cardiology Team & 1 \\
RAZRIV & Myocardial rupture & S\_AD\_KBRIG & Systolic blood pressure according to Emergency Cardiology Team & 1 \\
ZSN & Chronic heart failure & D\_AD\_KBRIG & Diastolic blood pressure according to Emergency Cardiology Team & 1 \\
IBS\_NASL & Heredity on CHD & D\_AD\_KBRIG & Diastolic blood pressure according to Emergency Cardiology Team & 1 \\
B\_BLOK\_S\_n & Use of beta-blockers in the ICU & S\_AD\_KBRIG & Systolic blood pressure according to Emergency Cardiology Team & 1 \\
B\_BLOK\_S\_n & Use of beta-blockers in the ICU & D\_AD\_KBRIG & Diastolic blood pressure according to Emergency Cardiology Team & 1 \\
lat\_im & Presence of a lateral myocardial infarction (left ventricular)  & ant\_im & Presence of an anterior myocardial infarction (left ventricular)  & 1 \\
K\_SH\_POST & Cardiogenic shock at the time of admission to intensive care unit & S\_AD\_KBRIG & Systolic blood pressure according to Emergency Cardiology Team & 1 \\
TRENT\_S\_n & Use of Trental in the ICU & D\_AD\_KBRIG & Diastolic blood pressure according to Emergency Cardiology Team & 1 \\
ritm\_ecg\_p\_07 & ECG rhythm at the time of admission to hospital: sinus with a heart rate above 90 (tachycardia) & D\_AD\_KBRIG & Diastolic blood pressure according to Emergency Cardiology Team & 1 \\
GB & Presence of an essential hypertension & D\_AD\_KBRIG & Diastolic blood pressure according to Emergency Cardiology Team & 1 \\
ASP\_S\_n & Use of acetylsalicylic acid in the ICU & D\_AD\_KBRIG & Diastolic blood pressure according to Emergency Cardiology Team & 1 \\
GIPO\_K & Hypokalemia ( < 4 mmol/L) & D\_AD\_KBRIG & Diastolic blood pressure according to Emergency Cardiology Team & 1 \\
P\_IM\_STEN & Post-infarction angina & D\_AD\_KBRIG & Diastolic blood pressure according to Emergency Cardiology Team & 1 \\
lat\_im & Presence of a lateral myocardial infarction (left ventricular)  & D\_AD\_KBRIG & Diastolic blood pressure according to Emergency Cardiology Team & 1 \\
lat\_im & Presence of a lateral myocardial infarction (left ventricular)  & ant\_im & Presence of an anterior myocardial infarction (left ventricular)  & 1 \\
FK\_STENOK & Functional class (FC) of angina pectoris in the last year. & S\_AD\_KBRIG & Systolic blood pressure according to Emergency Cardiology Team & 1 \\
FK\_STENOK & Functional class (FC) of angina pectoris in the last year. & D\_AD\_KBRIG & Diastolic blood pressure according to Emergency Cardiology Team & 1 \\
MP\_TP\_POST & Paroxysms of atrial fibrillation at the time of admission to intensive care unit, (or at a pre-hospital stage) & D\_AD\_KBRIG & Diastolic blood pressure according to Emergency Cardiology Team & 1 \\
FIBR\_PREDS & Atrial fibrillation & D\_AD\_KBRIG & Diastolic blood pressure according to Emergency Cardiology Team & 1 \\
GEPAR\_S\_n & Use of an anticoagulants (heparin) in the ICU & S\_AD\_KBRIG & Systolic blood pressure according to Emergency Cardiology Team & 1 \\
K\_SH\_POST & Cardiogenic shock at the time of admission to intensive care unit & S\_AD\_KBRIG & Systolic blood pressure according to Emergency Cardiology Team & 1 \\
ZSN\_A & Presence of chronic Heart failure (HF) in the anamnesis & S\_AD\_KBRIG & Systolic blood pressure according to Emergency Cardiology Team & 1 \\
B\_BLOK\_S\_n & Use of beta-blockers in the ICU & S\_AD\_KBRIG & Systolic blood pressure according to Emergency Cardiology Team & 1 \\
D\_AD\_KBRIG & Diastolic blood pressure according to Emergency Cardiology Team & S\_AD\_KBRIG & Systolic blood pressure according to Emergency Cardiology Team & 1 \\
RAZRIV & Myocardial rupture & S\_AD\_KBRIG & Systolic blood pressure according to Emergency Cardiology Team & 1 \\
FK\_STENOK & Functional class (FC) of angina pectoris in the last year. & S\_AD\_KBRIG & Systolic blood pressure according to Emergency Cardiology Team & 1 \\
GEPAR\_S\_n & Use of an anticoagulants (heparin) in the ICU & S\_AD\_KBRIG & Systolic blood pressure according to Emergency Cardiology Team & 1 \\
IBS\_POST & Coronary heart disease (CHD) in recent weeks, days before admission to hospital & D\_AD\_KBRIG & Diastolic blood pressure according to Emergency Cardiology Team & 1 \\
FIBR\_PREDS & Atrial fibrillation & D\_AD\_KBRIG & Diastolic blood pressure according to Emergency Cardiology Team & 1 \\
NA\_R\_2\_n & Use of opioid drugs in the ICU in the second day of the hospital period & D\_AD\_KBRIG & Diastolic blood pressure according to Emergency Cardiology Team & 1 \\
ASP\_S\_n & Use of acetylsalicylic acid in the ICU & D\_AD\_KBRIG & Diastolic blood pressure according to Emergency Cardiology Team & 1 \\
NA\_R\_1\_n & Use of opioid drugs in the ICU in the first hours of the hospital period & D\_AD\_KBRIG & Diastolic blood pressure according to Emergency Cardiology Team & 1 \\
MP\_TP\_POST & Paroxysms of atrial fibrillation at the time of admission to intensive care unit, (or at a pre-hospital stage) & D\_AD\_KBRIG & Diastolic blood pressure according to Emergency Cardiology Team & 1 \\
B\_BLOK\_S\_n & Use of beta-blockers in the ICU & D\_AD\_KBRIG & Diastolic blood pressure according to Emergency Cardiology Team & 1 \\
P\_IM\_STEN & Post-infarction angina & D\_AD\_KBRIG & Diastolic blood pressure according to Emergency Cardiology Team & 1 \\
ritm\_ecg\_p\_07 & ECG rhythm at the time of admission to hospital: sinus with a heart rate above 90 (tachycardia) & D\_AD\_KBRIG & Diastolic blood pressure according to Emergency Cardiology Team & 1 \\
lat\_im & Presence of a lateral myocardial infarction (left ventricular)  & D\_AD\_KBRIG & Diastolic blood pressure according to Emergency Cardiology Team & 1 \\
D\_AD\_KBRIG & Diastolic blood pressure according to Emergency Cardiology Team & S\_AD\_KBRIG & Systolic blood pressure according to Emergency Cardiology Team & 1 \\
\end{longtable}

\end{document}